\documentclass{article}

\PassOptionsToPackage{numbers, compress}{natbib}
\usepackage[main, final]{neurips_2026} 

\usepackage[utf8]{inputenc} 
\usepackage[T1]{fontenc}    
\usepackage{hyperref}       
\usepackage{url}            
\usepackage{booktabs}       
\usepackage{amsfonts}       
\usepackage{nicefrac}       
\usepackage{microtype}      
\usepackage{xcolor}         

\usepackage{amsmath} 
\usepackage{amsthm}
\usepackage{algorithm}
\usepackage{algpseudocode}
\usepackage{graphicx} 
\usepackage{float}
\usepackage{wrapfig}

\theoremstyle{plain} 
\newtheorem{theorem}{Theorem}

\newtheorem{proposition}[theorem]{Proposition}

\theoremstyle{definition} 

\theoremstyle{remark} 
\newtheorem{remark}{Remark}

\title{Fractional State Space Transition for Long Sequence Modeling} %

\author{Ivan Kobyzev\thanks{Equal contribution}\hspace{3mm}Abbas Ghaddar\footnotemark[1]\hspace{3mm}Ali Nasiri-Sarvi\hspace{3mm}Lifeng Shang\hspace{3mm}Yufei Cui \\
Huawei Noah’s Ark Lab, Montreal Research Center, Canada\\
{\footnotesize\texttt{\{ivan.kobyzev,abbas.ghaddar,shang.lifeng,yufei.cui\}@huawei.com}}}

\newcommand{\fracssm}{\textsc{Frac}}
\newcommand{\gdn}{\textsc{GDN}}
\newcommand{\mambatwo}{\textsc{Mamba2}}
\newcommand{\mambathree}{\textsc{Mamba3}}

\usepackage{amsfonts}

\usepackage{wrapfig}
\usepackage{caption}
\usepackage{placeins} 
\usepackage{multirow}

\begin{document}

\maketitle

\begin{abstract}
State Space Models (SSMs) compress sequence history into a bounded recurrent state, making the resulting memory law a central architectural choice for long-context performance. Most modern SSMs rely on ODE-based dynamics that lead to exponential forgetting, limiting their ability to retain information over broad temporal ranges. We introduce \fracssm{}, a selective SSM architecture derived from fractional dynamics that replaces this exponential decay with power-law long memory. To make fractional dynamics practical, \fracssm{} approximates the heavy-tailed target  kernel with a finite-state, log-spaced sum of exponential modes. This construction turns fractional memory into an efficient recurrent module with parallel training and prefill, while retaining bounded-state autoregressive decoding. Extensive experiments, including 1.3B-parameter language modeling, demonstrate that \fracssm{} consistently improves long-context performance over state-of-the-art SSM baselines while staying competitive on short-context. 
These results show that fractional dynamics provide a practical and effective prior for long-context SSMs. Code: \url{https://github.com/anasiri/frac-ssm}
\end{abstract}

\section{Introduction}

Due to the quadratic complexity of the Transformer self-attention module~\cite{vaswani2017attention}, a line of research has focused on developing more efficient linear alternatives based on State Space Models (SSMs)~\cite{gu2022s4,yang2024gla,mamba2,yang2024delta}. SSMs efficiency is based on compressing the entire sequence history into a finite recurrent state. This compression constraint makes the model’s internal memory a central design choice for long-context sequence modeling, which can be naturally understood through the memory law induced by the underlying dynamics. 

Most SSMs, such as Mamba~\citep{gu2024mamba,mamba2,lahoti2026mamba3}, base their internal memory design on ordinary differential equations (ODEs) and their discretizations~\citep{gu2022s4,hasani2023liquid}. Despite their success and widespread adoption, the ODE dynamics underlying these models naturally induce exponential forgetting, causing the influence of past inputs to decay exponentially with lag~\citep{wang2023ssmexpmemory}. That inductive bias is not optimal for modeling long sequences in which past events can retain non-negligible influence over broad temporal ranges. Fractional differential equations (FDEs) have been widely used in physics and applied mathematics to model such hereditary effects, as their solutions depend on the full history through heavy-tailed kernels~\citep{Diethelm2010,Mainardi2010,MetzlerKlafter2004,MainardiGorenflo2007}. Thus, FDEs can offer an alternative foundation for designing the internal memory of SSMs, rather than refining information selection or routing within exponentially forgetting ODE-based recurrent models~\citep{yang2025gdn}. However, a major challenge in making FDEs practical for SSMs is that, unlike ODEs, FDEs are non-Markovian: the state at a given time depends on the entire past, making them not directly compatible with finite-state recurrent layers.

In this paper, we introduce \fracssm{}, a novel selective SSM architecture derived from fractional dynamics. We develop a theoretical framework that makes fractional long memory compatible with efficient recurrent computation by approximating the target power-law kernel with a finite-state, log-spaced sum of exponential modes. The resulting \fracssm{} layer retains the computational advantages of modern selective SSMs, while also supporting hardware-efficient parallel training, prefill, and bounded-state autoregressive decoding.

On synthetic long-tail and recall-retrieval benchmarks, \fracssm{} achieves the strongest length extrapolation and highest recall performance among state-of-the-art SSMs, consistent with its intended heavy-tailed inductive bias and long-memory behavior. Furthermore, when training 1.3B-parameter language models from scratch, \fracssm{} outperforms strong SSM baselines such as Mamba and Gated DeltaNet (GDN)~\citep{yang2025gdn} on long-context evaluations, while remaining competitive on standard short-context language-modeling evaluations. Taken together, these results suggest that the memory law induced by the underlying dynamics is a powerful architectural axis for designing efficient long-context sequence models.

\section{Preliminaries}
\label{sec:background}

In this section, we review how dynamical systems with long memory can be modeled using FDEs, then we develop the theoretical framework that underlies the design of our \fracssm{} model in Section~\ref{sec:method}.

\subsection{Long memory in dynamical systems}
\label{subsec:physics_motivation}

Classical state-space models can be viewed as discrete-time counterparts of ordinary differential equations. In this setting, memory typically decays exponentially, so the influence of past events is controlled by a characteristic timescale. This is a natural model for many Markovian systems, where the future depends on the present state alone.
However, many systems in nature do not behave this way. In anomalous diffusion, dielectric relaxation, viscoelasticity, and related transport phenomena, the present state depends on a broad and weighted history of the past rather than on a single characteristic timescale~\citep{MetzlerKlafter2004,MainardiGorenflo2007}. Such systems exhibit long-memory or hereditary behavior, often described by kernels with heavy tails.
Fractional calculus~\citep{Diethelm2010,Mainardi2010} provides a convenient mathematical language for this regime. By allowing derivatives of non-integer order, it describes dynamics in which memory is distributed over the past rather than concentrated around a single exponential decay law.

\subsection{Caputo fractional differential equations}

To make the discussion above formal, let us start with the familiar first-order linear dynamics:
\begin{equation}
\dot{h}(t) = -\lambda h(t) + u(t),
\qquad
\lambda > 0.
\label{eq:ode-relaxation}
\end{equation}
This is the continuous-time analogue of a simple state-space update. Here the evolution of the state $h(t)$ is determined by the external input $u(t)$ together with the linear term $-\lambda h(t)$, which causes past information to decay at a rate set by $\lambda$. In this sense, the system is organized around a characteristic timescale of order $\lambda^{-1}$, which controls how quickly the influence of the past fades. As discussed in Section~\ref{subsec:physics_motivation}, this picture is no longer adequate when memory is distributed over the past rather than concentrated around a single scale.

A standard way to model this regime is to replace the ordinary derivative by a fractional derivative. In this paper, we use the Caputo fractional derivative~\citep{Diethelm2010}. For $0<\alpha<1$, it is defined by
\begin{equation}
{}^{C}D_{t}^{\alpha} f(t)
=
\frac{1}{\Gamma(1-\alpha)}
\int_{0}^{t}
\frac{f'(\tau)}{(t-\tau)^{\alpha}}
\, d\tau .
\label{eq:caputo-def}
\end{equation}
This expression makes the long-memory mechanism explicit: the derivative at time $t$ depends on the entire past history on $[0,t]$, with recent increments weighted more strongly but with a long algebraic tail that keeps older events relevant.
Replacing the ordinary derivative in \eqref{eq:ode-relaxation} by the Caputo derivative gives the fractional relaxation equation:
\begin{equation}
{}^{C}D_{t}^{\alpha} h(t)
=
-\lambda h(t) + u(t),
\qquad
\lambda > 0 ,
\label{eq:fractional-relaxation}
\end{equation}
which we use as the basic continuous-time model of long-memory dynamics. When $\alpha=1$, \eqref{eq:fractional-relaxation} reduces to the ordinary first-order system \eqref{eq:ode-relaxation}. For $0<\alpha<1$, the dynamics become nonlocal in time and the effective memory kernel becomes heavy-tailed. The parameter $\alpha$ therefore controls the strength of the long-memory effect: values closer to $1$ recover more local, ODE-like behavior, while smaller values produce slower forgetting and broader memory across timescales \citep{MainardiGorenflo2007}.

Eq.~\eqref{eq:fractional-relaxation} gives the right continuous-time model of long-memory dynamics, but it is not yet in the form needed for a state-space layer. Direct discretizations of FDEs are typically history-dependent: the update at time $t$ depends on the entire past trajectory rather than on a finite-dimensional recurrent state~\citep{Diethelm2010}. Rather than using a generic black-box fractional solver, we will construct a Markovian approximation of this dynamics that is compatible with efficient recurrent computation. Next, we develop this construction by expressing the fractional relaxation kernel in terms of Mittag--Leffler functions and then lifting it into a representation that admits a finite state-space realization.

\subsection{Mittag--Leffler relaxation and fractional kernels}

To understand what kind of memory law is induced by \eqref{eq:fractional-relaxation}, we now ask the same question one would ask for an ordinary linear system: how does the state decay in the absence of input, and how does it respond to an external forcing? In the classical case, both are governed by the exponential function. In the fractional case, the corresponding role is played by the Mittag--Leffler function~\citep{gorenflo2014}.

We first consider the homogeneous version of \eqref{eq:fractional-relaxation}: ${}^{C}D_{t}^{\alpha} h(t)
=
-\lambda h(t).$ Its solution is: $h(t)
=
E_{\alpha}(-\lambda t^{\alpha})\, h(0),$ where
\begin{equation}
\label{eq:Ea}
E_{\alpha}(z)
:=
\sum_{k=0}^{\infty}
\frac{z^{k}}{\Gamma(\alpha k + 1)}
\end{equation}
is the one-parameter Mittag--Leffler function \citep{LuchkoGorenflo1999}. Thus, in the fractional setting, the exponential decay law of the ordinary system is replaced by Mittag--Leffler relaxation.
We next return to the forced system \eqref{eq:fractional-relaxation}. Its response to the input is described by the impulse-response kernel

\begin{equation}
\label{eq:frac-impulse}
g_{\alpha,\lambda}(t)
:=
t^{\alpha-1}
E_{\alpha,\alpha}(-\lambda t^{\alpha}),
\qquad
\text{where}
\quad
E_{\alpha,\beta}(z)
:=
\sum_{k=0}^{\infty}
\frac{z^{k}}{\Gamma(\alpha k + \beta)} 
\end{equation}

is the two-parameter Mittag--Leffler function. The full solution to the FDE \eqref{eq:fractional-relaxation} can be written as:
\begin{equation}
h(t)
=
E_{\alpha}(-\lambda t^{\alpha})\, h(0)
+
\int_{0}^{t} g_{\alpha,\lambda}(t-\xi)\, u(\xi)\, d\xi ,
\label{eq:frac-full-solution}
\end{equation}
so the kernel $g_{\alpha,\lambda}$ describes how past inputs are accumulated over time~\citep{LuchkoGorenflo1999}.

\begin{remark}
Fractional relaxation can be informally associated with power-law memory. More precisely, the relevant Mittag--Leffler kernels are not exact power laws for all \(t\), but they do have asymptotically polynomial tails as \(t\to\infty\). 
For any fixed \(\alpha\in(0,1)\) and \(\lambda>0\), the Mittag--Leffler homogeneous kernel satisfies~\citep{gorenflo2014}:
\begin{equation}
E_{\alpha}(-\lambda t^\alpha)
=
\frac{1}{\lambda\,\Gamma(1-\alpha)}\, t^{-\alpha}
+
O(t^{-2\alpha}),
\qquad t\to\infty.
\label{eq:mittag-tail}
\end{equation}
Likewise, the impulse-response kernel satisfies~\citep{gorenflo2014}:
\begin{equation}
g_{\alpha,\lambda}(t)
=
t^{\alpha-1}E_{\alpha,\alpha}(-\lambda t^\alpha)
=
\frac{\alpha}{\lambda^2\Gamma(1-\alpha)}\, t^{-\alpha-1}
+
O(t^{-2\alpha-1}),
\qquad t\to\infty.
\label{eq:gal-tail}
\end{equation}
In particular, the homogeneous solution decays as \(t^{-\alpha}\), in contrast to the exponential decay of the ordinary first-order system. This asymptotically polynomial tail is the source of the broader timescale coverage that motivates the fractional construction for long-memory sequence modeling.
\end{remark}

At the same time, the solution \eqref{eq:frac-full-solution} is still not in the finite-dimensional Markovian form needed for an efficient state-space layer. The next subsection addresses this issue by rewriting the fractional kernel in a form that can be approximated by a finite bank of exponential modes.

\subsection{Diffusive representations of fractional kernels}

The key structural fact we need is that the fractional kernel admits a representation as a continuous mixture of ordinary exponential decays.

\begin{theorem}
\label{thm:diffusive-representation}
For $0<\alpha < 1$ and $\lambda > 0$, the kernels appearing in the solution \eqref{eq:frac-full-solution} of the fractional differential equation \eqref{eq:fractional-relaxation} admit  nonnegative diffusive representations. In particular, there exist nonnegative densities $R_{\alpha, \lambda}$ and $H_{\alpha,\lambda}$ such that

\begin{equation}
\label{eq:Ea-diffusive}
\begin{aligned}
E_{\alpha}(-\lambda t^{\alpha})
&=
\int_{0}^{\infty}
e^{-t/\tau}\, R_{\alpha,\lambda}(\tau)\, d\tau,
\qquad
g_{\alpha,\lambda}(t)
&=
\int_{0}^{\infty}
e^{-t/\tau}\, H_{\alpha,\lambda}(\tau)\, d\tau .
\end{aligned}
\end{equation}

\end{theorem}

See Appendix~\ref{app:diffusive-proof} for more details. 
This representation is still infinite-dimensional. The next subsection shows how to approximate it on a bounded horizon by a finite bank of exponential modes, which is the form we will later turn into a discrete state-space transition.

\subsection{Finite sum-of-exponentials (SoE) approximation on a bounded horizon}
\label{subsec:soe-background}

The diffusive representation of Theorem~\ref{thm:diffusive-representation} still involves a continuum of timescales and therefore cannot be used directly as a finite-state recurrent transition. To obtain a finite memory bank, we approximate this integral on the bounded range of timescales relevant for the horizon of interest by a finite weighted sum of exponential modes. Because fractional kernels spread mass across many orders of magnitude in time, this construction is naturally organized on a logarithmic timescale grid. Related exponential-sum constructions of this type are standard in numerical methods for fractional kernels \citep{JiangZhangZhang2017,Chaudhary2025}. The next theorem states the exact structural facts from this approximation that is the key component for our model in Section~\ref{sec:method}.

\begin{theorem}
\label{thm:soe-approx}
Fix \(0<\alpha<1\), \(\lambda>0\), and \(0<T<\infty\). Then for every \(\varepsilon>0\)
there exist \(M\in\mathbb N\),
\(\tau_0>0\), and \(q>1\),
a geometrically spaced bank of positive timescales $ \tau_m=\tau_0 q^{m-1}$ for $m=1,\dots,M$, 
and positive coefficients \(c_m(\alpha,\lambda)\),
\(d_m(\alpha,\lambda)\) such that
\begin{align}
\sup_{t\in[0,T]}
\left|
E_{\alpha}(-\lambda t^\alpha)
-
\sum_{m=1}^{M} c_m(\alpha,\lambda)e^{-t/\tau_m}
\right|
&\le \varepsilon,
\label{eq:shared-soe-homogeneous}
\\
\int_0^T
\left|
g_{\alpha,\lambda}(t)
-
\sum_{m=1}^{M} d_m(\alpha,\lambda)e^{-t/\tau_m}
\right|dt
&\le \varepsilon.
\label{eq:shared-soe-forced}
\end{align}
Moreover, the coefficients may be chosen so that
\[
\frac{d_m(\alpha,\lambda)}{c_m(\alpha,\lambda)}
=
\frac{1}{\lambda\tau_m},
\qquad m=1,\dots,M,
\]
and, for the large-timescale part of the geometric bank $\tau_m$:
\begin{equation} 
\label{eq:asymptotics}
c_m(\alpha,\lambda)
=
\frac{(\log q)\sin(\pi\alpha)}{\pi\lambda}\,
\tau_m^{-\alpha}\bigl(1+O(\tau_m^{-\alpha})\bigr).
\end{equation}
\end{theorem}

\noindent\textit{Proof.} See Appendix~\ref{app:soe-proof}.

\section{Method}
\label{sec:method}

We now turn the continuous-time fractional memory picture of Section~\ref{sec:background} into the discrete selective recurrence underlying the \textsc{Frac} layer, our selective fractional state-space module. The construction has three steps. First, we approximate the target long-memory kernel by a finite bank of exponential modes. Second, we discretize this mode bank exactly under a zero-order-hold (ZOH) assumption. Third, we make the resulting transition selective through token-dependent control variables and mode-wise read and write weights.

\subsection{From finite SoE kernels to a memory bank}
\label{subsec:soe-to-bank}

Consider the shared-bank SoE approximation from Theorem~\ref{thm:soe-approx}
with \(M\) memory modes. We now show that, on the horizon of interest, the
fractional dynamical system is approximated by a bank of \(M\) first-order ODEs.

\begin{proposition}
\label{prop:soe-realization}
Fix \(0<T<\infty\) and $\varepsilon>0$, and let
\(\{\tau_m,c_m(\alpha,\lambda),d_m(\alpha,\lambda)\}_{m=1}^M\)
be the common-bank coefficients given by Theorem~\ref{thm:soe-approx} on \([0,T]\) for this tolerance \(\varepsilon\).
Given \(u\in L^\infty([0,T])\), denote by \(h\) the solution on \([0,T]\) of
the fractional differential equation~\eqref{eq:fractional-relaxation} with input \(u\).

Consider the system of ODEs
\begin{equation}
\dot s_m(t)
=
-\frac{1}{\tau_m}s_m(t)
+
a_m(\alpha,\lambda)\,u(t),
\qquad m=1,\dots,M,
\label{eq:memory-bank-ode}
\end{equation}
with initial conditions $s_m(0)=h(0)$ for $m=1,\dots,M.$ 
If
\begin{equation}
a_m(\alpha,\lambda):=\frac{d_m(\alpha,\lambda)}{c_m(\alpha,\lambda)} = \frac{1}{\lambda\tau_m},
\qquad m=1,\dots,M,
\label{eq:write-coefficient}
\end{equation}
and if the bank readout is defined by
\begin{equation}
\tilde h_M(t):=\sum_{m=1}^{M} c_m(\alpha,\lambda)\,s_m(t),
\label{eq:memory-bank-readout}
\end{equation}
then
\begin{equation}
\sup_{t\in[0,T]}
|h(t)-\tilde h_M(t)|
\le
\varepsilon\Bigl(|h(0)|+\|u\|_{L^\infty([0,T])}\Bigr).
\label{eq:soe-realization-error}
\end{equation}
In particular, \(\tilde h_M\) approximates \(h\) uniformly on \([0,T]\).
\end{proposition} 

For the proof see Appendix~\ref{app:soe-realization-proof}. This gives a finite continuous-time realization of the target long-memory kernel. Next,  we discretize this mode bank and turn it into the selective recurrent transition.

\subsection{Frac State Transition}
\label{subsec:frac-transition}
We now convert the finite continuous-time mode bank of Proposition~\ref{prop:soe-realization} into a token-level recurrent transition by discretizing the mode dynamics under the standard zero-order-hold (ZOH)~\citep{gu2024mamba} procedure. On each interval, the control variables are frozen and the input is held constant, so each mode evolves as a scalar first-order linear system. Let $u_t$ denote the held input and let $\Delta_t>0$ denote the interval length. Applying the exact ZOH discretization to \eqref{eq:memory-bank-ode} on the interval $[0,\Delta_t)$, one gets:
\begin{equation}
s_{t,m}
=
\rho_{t,m} s_{t-1,m} + \beta_{t,m} u_t,
\label{eq:zoh-recurrence}
\end{equation}
where
\begin{equation}
\rho_{t,m} = \exp(-\Delta_t / \tau_m),
\qquad
\beta_{t,m} = \frac{1-\rho_{t,m}}{\lambda}.
\label{eq:rho-beta-def}
\end{equation}
Thus \(\rho_{t,m}\) determines how much of the previous state is retained over the interval, while \(\beta_{t,m}\) is the exact ZOH injection factor.
To specialize this recurrence to the fractional setting, we use the self-similarity in \(\lambda\) from Remark~\ref{rem:self-similarity}. That remark shows that varying \(\lambda\) rescales the underlying timescale axis by the factor \(\lambda^{-1/\alpha}\). We keep a shared geometric bank of base timescales \(\{\tau_m\}_{m=1}^M\) and implement this effect through the token-dependent effective timescale
\begin{equation}
\tilde{\tau}_{t,m}
=
\frac{\tau_m}{\lambda_t^{1/\alpha_t}}.
\label{eq:tau-eff}
\end{equation}
Substituting \(\tilde{\tau}_{t,m}\) for \(\tau_m\) in \eqref{eq:rho-beta-def} gives
\begin{equation}
\tilde\rho_{t,m}
=
\exp\!\left(-\Delta_t \lambda_t^{1/\alpha_t} / \tau_m\right),
\qquad
\tilde\beta_{t,m}
=
\frac{1-\tilde\rho_{t,m}}{\lambda_t}.
\label{eq:rho-beta-def-frac}
\end{equation}
Thus \(\Delta_t\), \(\alpha_t\), and \(\lambda_t\) determine both the
mode retention factors and the exact ZOH injection factors. The learned
routing introduced next provides additional content-dependent modulation
over this fractional transition.

The recurrence \eqref{eq:zoh-recurrence} gives the nonselective update of the mode bank. We introduce token-dependent write weights \(b_{t,m}\) and read weights \(c_{t,m}\) over the modes. The resulting selective update is:
\begin{align}
s_{t,m}
&=
\tilde{\rho}_{t,m} s_{t-1,m}
+
\kappa_{t,m} u_t,
\label{eq:frac-selective-update}
\\
h_t
&=
\sum_{m=1}^M c_{t,m} s_{t,m},
\label{eq:frac-readout}
\end{align}
where $ \kappa_{t,m} =
\tilde{\beta}_{t,m} b_{t,m}$.
Thus \(\tilde{\rho}_{t,m}\) and \(\tilde{\beta}_{t,m}\) set the transition rule of the mode bank, while \(b_{t,m}\) and \(c_{t,m}\) determine how the input is distributed across modes and how the updated bank is read out.

On the read side, the fractional theory gives the asymptotics \eqref{eq:asymptotics} for the readout coefficients \(c_m\). We therefore use the log-timescale prior \(-\alpha_t \log \tau_m\) and define
\begin{equation}
c_{t,m}
=
\operatorname{softmax}_m\!\left(
-\alpha_t \log \tau_m + g_{\mathrm{read}}(u_t)_m
\right),
\label{eq:read-softmax}
\end{equation}
where \(g_{\mathrm{read}}\) is a learned linear map from the content signal to mode-wise residual logits.

The write prior is not uniquely fixed by the theory. In the implementation used in this paper, however, we define \(b_{t,m}\) using the same functional form as in \eqref{eq:read-softmax}, but with an independent linear map \(g_{\mathrm{write}}\). This gives the read and write pathways a common log-timescale addressing scheme over the mode bank, which empirically leads to better alignment between where information is written and how it is later retrieved. 
See Algorithm~\ref{alg:frac-state-transition} in Appendix~\ref{app:frac_transition_algo} for the details.

\begin{remark}
The transition \eqref{eq:frac-selective-update} remains a first-order affine recurrence. For a fixed head and feature coordinate, stacking the \(M\) mode states into a vector \(\mathbf{s}_t\in\mathbb{R}^M\) gives
\begin{equation}
\mathbf{s}_t=\mathbf{\Lambda}_t\mathbf{s}_{t-1}+\mathbf{b}_t,
\qquad
\mathbf{\Lambda}_t=\mathrm{Diag}(\tilde\rho_{t,1},\dots,\tilde\rho_{t,M}),
\end{equation}
where \(\mathbf{b}_t\) collects the write terms \(\kappa_{t,m}u_t\). Because this is an affine recurrence, 
training and prefill can be implemented with a chunked parallel scan, following the standard scan composition used in modern selective state-space models~\citep{Blelloch1990,gu2024mamba}. Autoregressive decoding uses the same rule in one-step recurrent form with cached states. We implement this computation with a custom Triton kernel.\footnote{See Appendix~\ref{app:FracMixer Efficient Implementation} for additional implementation details and latency benchmarking.}
\end{remark}

\subsection{The \textsc{Frac} layer}
\label{subsec:frac-block}

\begin{wrapfigure}{r}{0.55\linewidth}
\vspace{-1.2em}
\centering
\includegraphics[width=\linewidth]{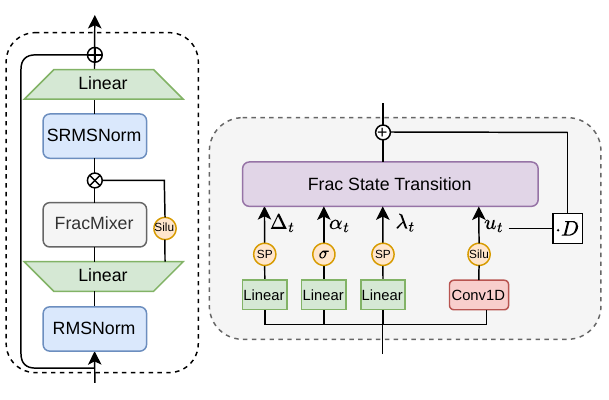}
\caption{Overview of the \fracssm{} layer (left), and \fracssm{Mixer} block (right).}
\label{fig:frac-block}
\vspace{-1.0em}
\end{wrapfigure}

Figure~\ref{fig:frac-block} (left) summarizes the design of \textsc{Frac} layer, which follows the common practice of recent linear-time sequence models~\citep{mamba2,yang2025gdn}. Starting from the hidden state, we first apply normalization, then an input projection, and pass the result to the FracMixer block. This returns the updated hidden representation, which is then processed by gated normalization and a final output projection. 

FracMixer is shown in Figure~\ref{fig:frac-block} (right). Following the common design pattern of recent linear models~\citep{mamba2,yang2025gdn}, the input is split into a pre-convolution (control) branch and a post-convolution (content) branch. The control branch produces the token-wise variables \(\Delta_t\), \(\alpha_t\), and \(\lambda_t\). As in Mamba~\citep{gu2024mamba}, \(\Delta_t\) is obtained from a linear projection with learned bias followed by a softplus (SP) transform, 
while \(\alpha_t\) and \(\lambda_t\) are produced by separate linear projections followed by pointwise nonlinearities: sigmoid for $\alpha$ and SP for $\lambda$
enforcing \(0<\alpha_t<1\) and  \(\lambda_t>0\). The content branch produces \(u_t\) by applying a local depthwise causal 1D convolution along the sequence. The Frac state transition module then applies the selective fractional recurrence across the sequence, as described in Section~\ref{subsec:frac-transition}. Its output is combined with the direct feedthrough term \(D u_t\), where \(D\) is a learnable parameter, mirroring the standard direct feedthrough term in classical state-space models~\citep{gu2022s4}.

\section{Experiments}
\label{sec:experiments}

\subsection{Synthetic Benchmarks}

\paragraph{Heavy Tail Probing}
As a controlled probe of the memory law, we introduce a simple synthetic extrapolation task where sparse events must be accumulated with a fixed power-law decay over distance. We compare the \fracssm{Mixer} block with its counterparts from prior linear-modeling work, including \mambatwo{}~\citep{mamba2}, \gdn{}~\citep{yang2025gdn}, and \mambathree{}~\citep{lahoti2026mamba3}, as well as with vanilla self-attention~\citep{vaswani2017attention}.

All models use a single layer with 200K parameters, are trained only on sequences of length 512, and are evaluated on sequences up to 128K tokens.\footnote{Models and additional implementation details are presented in Appendix~\ref{app:heavy-tail-synth}.} Figure~\ref{fig:synthetic_longtail} shows that while all models degrade as the test context length increases, \fracssm{} consistently exhibits the smallest performance decay. While \gdn{} and \mambathree{} remain the closest competitors to our model, Attention rapidly drops to near-random performance starting at $8$K.
These results suggest that changing the memory law itself can lead to substantially better length generalization.

\begin{wrapfigure}[24]{r}{0.56\linewidth}
\vspace{-4mm}
\centering

\includegraphics[width=\linewidth]{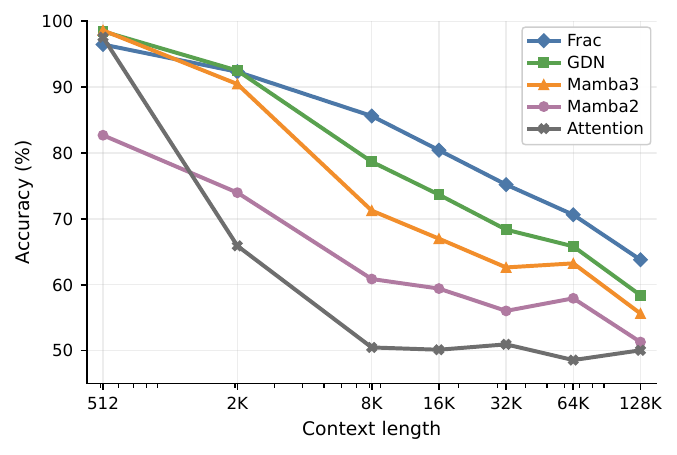}
\caption{Model performance on heavy-tail synthetic extrapolation task.}
\label{fig:synthetic_longtail}

\vspace{3mm}

\small
\setlength{\tabcolsep}{3.2pt}
\renewcommand{\arraystretch}{0.95}
\resizebox{\linewidth}{!}{
\begin{tabular}{lcccc|c}
\toprule
Model       & Comp. & F-ICR & Mem.   & SC    & Avg. \\
\midrule
\mambatwo{}      & $88.3{\scriptscriptstyle \pm 0.2}$ & $7.2{\scriptscriptstyle \pm 0.3}$  & $99.9{\scriptscriptstyle \pm 0.1}$ & $99.9{\scriptscriptstyle \pm 0.1}$ & 73.8 \\
\gdn{}           & $88.9{\scriptscriptstyle \pm 0.1}$ & $9.7{\scriptscriptstyle \pm 0.2}$  & $98.8{\scriptscriptstyle \pm 0.1}$ & $99.9{\scriptscriptstyle \pm 0.1}$ & 74.3 \\
\mambathree{}    & $89.1{\scriptscriptstyle \pm 0.2}$ & $10.1{\scriptscriptstyle \pm 0.4}$ & \bf $100.0{\scriptscriptstyle \pm 0.0}$ & \bf $100.0{\scriptscriptstyle \pm 0.0}$ & 74.8 \\
\fracssm{}        & \bf $90.4{\scriptscriptstyle \pm 0.2}$ & \bf $11.5{\scriptscriptstyle \pm 0.2}$ & $99.9{\scriptscriptstyle \pm 0.1}$ & $99.9{\scriptscriptstyle \pm 0.1}$ & 75.4 \\
\bottomrule
\end{tabular}
}

\captionof{table}{Performance (accuracy) on the synthetic MADLab benchmark over 5 runs.}
\label{tab:madlab_main}
\vspace{-10mm}
\end{wrapfigure}
\paragraph{MADLab} We next evaluate on MADLab~\citep{poli24_madlab}, a suite of synthetic tasks that probe sequence-modeling mechanisms including compression (Comp.), fuzzy in-context recall (F-ICR), memorization (Mem.), and selective copying (SC). All models are trained from scratch and have four layers with roughly 500K parameters, alternating sequence-mixing layers and SwiGLU channel-mixing layers.\footnote{MADLab also includes plain ICR and noisy ICR tasks, which we do not report because performance saturates for all models. See Appendix~\ref{app:madlab} for additional implementation details.
} As shown in Table~\ref{tab:madlab_main}, \fracssm{} is on par with the strongest linear baselines and slightly improves the overall average. These results suggest that replacing the standard ODE-based memory law with an FDE-based one preserves the core mechanistic abilities of SSMs.

\subsection{Language Modeling}

\paragraph{Setup} We compare \fracssm{} with state-of-the-art models, including the linear models \mambatwo{}~\cite{mamba2}, 
\gdn{}~\citep{yang2025gdn}, and \mambathree{}~\cite{lahoti2026mamba3} (both \textit{-SISO} and \textit{-MIMO} variants), as well as Vanilla Transformer~\cite{yang2025qwen3}. Following~\citep{yang2025gdn}, we pretrain 1.3B-parameter LLMs from scratch on 100B tokens sampled from the deduplicated FineWeb-Edu~\cite{penedo2024the,benallal2024smollmcorpus} corpus, training all models under the same standard protocol. We use the Llama-2~\citep{arxiv23_llama2} tokenizer with a 32k-token vocabulary and perform training on fully packed sequences with a length of 4k. Following prior work~\citep{mamba2,yang2024delta,lahoti2026mamba3}, we evaluate models on long-context needle-in-a-haystack tasks~\citep{needle} (\textbf{NIAH}) and \textbf{LongBench}~\citep{bai2024longbench}, as well as short-context language modeling benchmarks (\textbf{LM Harness}) and real-world intensive recall-retrieval tasks~\citep{arora-2024-jrt} (\textbf{Recall-Retrieval}). A detailed description of the training, implementation, and evaluation protocols is provided in Appendix~\ref{app:Language Modeling}.

\noindent \textbf{NIAH} 
Unlike prior works, we evaluate NIAH far beyond the $4$K training sequence length, testing extrapolation up to $64$K tokens. Figure~\ref{fig:niah_main} shows that \fracssm{} performs on par with baseline linear models on short-context settings ($\leq 4$K), while demonstrating substantially stronger length generalization once the context exceeds the training range. In particular, as the sequence length increases from $8$K to $64$K, the performance of \fracssm{} drops more slowly than that of the baseline models. The main exception is GDN on \textit{S-NIAH-1}, where its gated delta rule is especially effective at filtering repetitive context~\citep{yang2025gdn}.

\begin{figure*}[!tph]
    \centering
    \includegraphics[width=\linewidth]{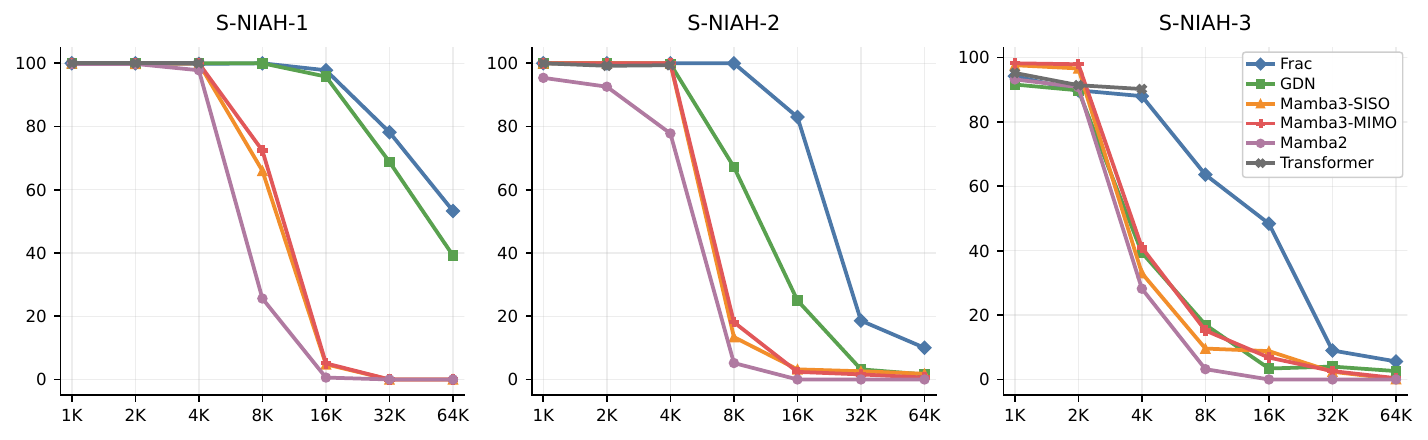}
    \caption{Model accuracies on three variants of the \textit{needle-in-a-haystack} passkey retrieval task, with sequence lengths ranging from $1$K to $64$K. Transformer performance is zero after $4$K.}
    \label{fig:niah_main}
\end{figure*}

\noindent\textbf{LongBench.}
Table~\ref{tab:longbench_main} shows that \fracssm{} improves long-context language processing tasks over both the Transformer and linear-model baselines. In particular, \fracssm{} outperforms the strongest linear baseline \gdn{} by 1.9\% on average, and reports the best result on 8/14 tasks. The gains are smaller than on NIAH, which is expected because more than half of LongBench tasks have average input lengths below $8$K tokens, so the Transformer remains competitive with linear recurrent models. Nevertheless, these results suggest that the fractional memory mechanism retains useful information over longer spans more effectively than standard exponential-decay recurrent models.

\begin{table*}[!t]
\centering
\resizebox{\linewidth}{!}{
\setlength{\tabcolsep}{4.0pt}
\scriptsize
\begin{tabular}{l|ccc|ccc|ccc|ccc|cc|c}
\toprule
& \multicolumn{3}{c|}{Single-Doc QA}
& \multicolumn{3}{c|}{Multi-Doc QA}
& \multicolumn{3}{c|}{Summarization}
& \multicolumn{3}{c|}{Few-shot}
& \multicolumn{2}{c|}{Code}
& \multirow{2}{*}{Avg.}\\
Model
& { NQA} & { QQA} & { MFQ}
& {HQA} & { 2WM} & { Mus}
& {GvR} & { QMS} & { MNs}
& {TRC} & { TQA} & { SSM}
& {LCC} & { RBP}
& \\
\midrule
Transformer & 8.7 & 10.6 & 19.0 & 6.1 & \underline{15.5} & 4.0 & 7.2 & 16.6 & 11.2 & 25.0 & 19.1 & 11.3 & \textbf{26.0} & \underline{25.6} & 14.7 \\
\mambatwo{} & 4.4 & 11.4 & 16.3 & 4.5 & 13.8 & 2.7 & 7.8 & 12.7 & 6.3 & 17.0 & 10.9 & 10.5 & \underline{23.4} & \textbf{26.5} & 12.0 \\
\gdn{} & \underline{11.5} & 10.1 & \underline{24.4} & 7.3 & 14.8 & 4.8 & 7.9 & \underline{18.9} & \underline{14.3} & \textbf{40.0} & \underline{21.6} & 13.2 & 18.7 & 16.9 & \underline{16.0} \\
\mambathree{\tiny{\textit{-SISO}}} & 11.2 & \underline{12.4} & 22.6 & \underline{8.4} & \textbf{16.3} & \textbf{6.4} & 9.8 & \textbf{19.3} & 9.6 & 13.5 & 18.2 & \underline{14.5} & 20.6 & 18.3 & 14.4 \\
\mambathree{\tiny{\textit{-MIMO}}} & 5.0 & 10.4 & 15.3 & 8.2 & 12.5 & \underline{5.6} & \underline{10.5} & 14.3 & 8.4 & 8.5 & 8.6 & 11.6 & 20.7 & 19.9 & 11.4 \\
\fracssm{} & \textbf{13.2} & \textbf{14.0} & \textbf{25.8} & \textbf{8.5} & 12.2 & 5.1 & \textbf{13.8} & \underline{18.9} & \textbf{16.3} & \underline{39.0} & \textbf{24.5} & \textbf{28.2} & 11.0 & 19.5 & \textbf{17.9} \\
\bottomrule
\end{tabular}
}
\caption{Model performance on 14 LongBench tasks grouped under 5 categories. The highest and second-highest scores are highlighted in bold and underline, respectively.}
\label{tab:longbench_main}
\end{table*}

\begin{table*}[!th]
\centering
\resizebox{\linewidth}{!}{
\setlength{\tabcolsep}{3.2pt}
\scriptsize
\begin{tabular}{l|cc|cccccccc|c}
\toprule
\textbf{Model}
& \textbf{Wiki.}
& \textbf{LMB.}
& \textbf{LMB.}
& \textbf{PIQA}
& \textbf{Hella.}
& \textbf{Wino.}
& \textbf{ARC-e}
& \textbf{ARC-c}
& \textbf{OBQA}
& \textbf{BoolQ}
& \textbf{Avg.} \\
& ppl $\downarrow$
& ppl $\downarrow$
& acc $\uparrow$
& acc $\uparrow$
& acc\_n $\uparrow$
& acc $\uparrow$
& acc $\uparrow$
& acc\_n $\uparrow$
& acc $\uparrow$
& acc $\uparrow$
& \\
\midrule
Transformer & \textbf{15.3} & 13.6 & \underline{45.6} & 72.1 & \underline{57.6} & \textbf{59.5} & 70.7 & 35.8 & 39.4 & \textbf{61.9} & \textbf{55.3} \\
\mambatwo{} & 16.7 & \underline{13.5} & 40.9 & 72.1 & 57.1 & 54.0 & 71.2 & \underline{38.3} & 40.3 & 56.2 & 53.8 \\
\gdn{} & 16.4 & 13.6 & 42.8 & 71.8 & 54.7 & \underline{57.3} & 70.5 & 37.5 & 39.3 & 59.9 & 54.2 \\
\mambathree{\tiny{\textit{-SISO}}} & 16.1 & 13.7 & 45.2 & \textbf{72.8} & 57.2 & 55.6 & \underline{72.3} & 36.3 & \underline{40.6} & 57.0 & 54.6 \\
\mambathree{\tiny{\textit{-MIMO}}} & \underline{15.9} & \textbf{12.3} & \textbf{47.3} & \textbf{72.8} & \textbf{57.9} & 55.9 & \textbf{73.2} & \textbf{39.2} & 38.6 & 57.5 & \bf 55.3 \\
\fracssm{} & 16.5 & \underline{13.5} & 44.9 & \underline{72.4} & 56.4 & 55.6 & 71.6 & 38.0 & \textbf{41.8} & \underline{60.0} & \underline{55.1} \\
\bottomrule
\end{tabular}
}
\caption{Model performance on short-context language modeling and understanding tasks. The best and second-best scores for each metric are highlighted in bold and underline, respectively. The average is computed by excluding the first two columns.}
\label{tab:lm_harness_results}
\vspace{-6mm}
\end{table*}

\paragraph{LM Harness}
On short-context language understanding and commonsense reasoning tasks, \fracssm{} achieves performance that is competitive with other linear-time baselines and the Transformer. As shown in Table~\ref{tab:lm_harness_results}, \fracssm{} is only 0.2\% behind the Transformer and \mambathree{\tiny{\textit{-MIMO}}} on average, while achieving perplexity comparable to the other models. Since these tasks typically involve short sequences of roughly $32$--$256$ tokens, the results indicate that \fracssm{} preserves short-context language understanding abilities despite its structural bias toward modeling heavy-tailed long-memory.
\begin{wraptable}[12]{r}{0.6\linewidth}
\vspace{-2mm}
\centering
\setlength{\tabcolsep}{3.2pt}
\renewcommand{\arraystretch}{0.92}
\scriptsize
\resizebox{\linewidth}{!}{
\begin{tabular}{l|cccccc|c}
\toprule
Model
& {\tiny SWDE} & {\tiny SQD} & {\tiny FDA} & {\tiny TQA} & {\tiny NQ} & {\tiny Drop}
& {\tiny Avg.} \\
\midrule
Transformer
& \underline{42.8} & 39.6 & \underline{27.4} & \textbf{65.7} & \textbf{27.3} & \textbf{28.3} & \textbf{38.5} \\
\mambatwo{}
& 32.8 & 31.0 & 13.6 & 60.5 & 20.4 & 20.8 & 29.9 \\
\gdn{}
& 32.7 & 35.1 & 12.2 & 58.4 & 22.1 & 24.6 & 30.9 \\
\mambathree{\tiny{\textit{-SISO}}}
& 32.2 & \underline{39.8} & 26.7 & 61.5 & 24.8 & 25.7 & 35.1 \\
\mambathree{\tiny{\textit{-MIMO}}}
& 34.5 & \textbf{41.2} & \textbf{28.9} & \underline{64.2} & 25.8 & 26.1 & \underline{36.8} \\
\fracssm{}
& \textbf{43.2} & 39.0 & 20.0 & 64.0 & \underline{25.9} & \underline{26.2} & 36.4 \\
\bottomrule
\end{tabular}
}
\caption{Model performance on 6 real-world recall-retrieval tasks. The highest and second-highest scores are highlighted in bold and underline, respectively. }
\label{tab:retrieval_results}

\end{wraptable}

\vspace{-8mm}
\noindent \paragraph{Recall-Retrieval}
A similar trend is observed on real-world recall- and retrieval-intensive tasks~\citep{arora-2024-jrt}, where \fracssm{} remains competitive with prior models. It ranks third overall, trailing the Transformer by 2.1\% on average and the second best \mambathree{\small{\textit{-MIMO}}} by only 0.4\%. It is worth noting that although the original tasks were designed to be challenging at very long sequences, we follow prior work~\citep{mamba2,lahoti2026mamba3,yang2025gdn} and evaluate on truncated 2K-token sequences, making them primarily short-context retrieval tasks in this setting. Nevertheless, \fracssm{} retains strong short-context retrieval ability despite being structurally designed for long-context modeling.

\subsection{DNA modeling}
\label{subsec:dna}

\begin{wrapfigure}[7]{r}{0.33\linewidth}
    \vspace{-6mm}
    \centering
    \includegraphics[height=2.9cm]{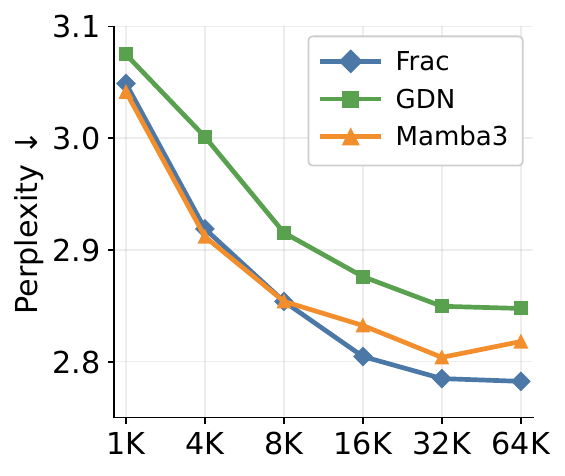}
    \vspace{-1.5mm}
    \caption{DNA perplexity vs. training sequence length.}
    \label{fig:dna}
    \vspace{-3mm}
\end{wrapfigure}

We evaluate \fracssm{} on genomic sequence modeling. Following HyenaDNA~\citep{nguyen2023hyenadna}, we train 7M-parameter causal language models on the human genome HG38 dataset across sequence lengths ranging from $1$K to $64$K, testing the scaling law.  Figure~\ref{fig:dna} shows that \fracssm{} achieves on par or lower perplexity than \mambathree{} and \gdn{} across all lengths, with larger gains at longer contexts. Full details are presented in Appendix~\ref{app:dna}.

\vspace{-2.5mm}
\section{Related Work}
\label{sec:related-work}
\paragraph{State-space sequence models.}
State-space models provide a principled route to efficient sequence modeling by viewing sequence layers as discretizations of continuous-time dynamical systems. Structured SSMs such as S4~\citep{gu2022s4} showed that parameterized linear ODEs can yield long-range sequence models with efficient convolutional or recurrent implementations, with subsequent work simplifying or extending this view through diagonal parameterizations~\citep{gu2022s4d} and input-dependent dynamics~\citep{hasani2023liquid}. Our work follows this continuous-time perspective, but changes the underlying memory law: rather than starting from an ODE with exponential memory, \fracssm{} starts from fractional dynamics and realizes the heavy-tail memory through a finite sum-of-exponentials approximation.

\paragraph{Selective SSMs and Mamba.}
Recent SSM language models improve discrete sequence modeling through input selectivity. Mamba~\citep{gu2024mamba} makes the transition and input/output projections token-dependent. \mambatwo{}~\citep{mamba2} develops the state-space duality view, connecting selective SSMs to semiseparable attention-like operators and enabling faster hardware-efficient implementations. Mamba-3~\citep{lahoti2026mamba3} further improves this line through a more expressive discretization, complex-valued state dynamics, and a multi-input multi-output (MIMO) recurrence. \fracssm{} explores a different axis: it retains the efficient selective recurrent structure but changes the memory kernel being discretized. In particular, the multiple modes in \fracssm{} should not be confused with Mamba-3's MIMO design. The modes in \fracssm{} arise as quadrature modes over timescales in a finite sum-of-exponentials approximation to a fractional kernel, rather than as a MIMO parameterization of the recurrent map.

\paragraph{Linear attention, associative updates, and multiple memories.}
Several efficient sequence models approach bounded-state recurrence from the linear-attention side.
RetNet~\citep{sun2023retnet} uses fixed decay factors across retention heads to cover different memory timescales, with each head using a single exponential decay.
Gated Linear Attention~\citep{yang2024gla} uses token-dependent forgetting, while delta-rule models~\citep{Schlag2021delta,yang2024delta} improve associative memory updates. Gated DeltaNet~\citep{yang2025gdn} combines gating with delta-rule updates to improve retrieval, length extrapolation, and long-context understanding. Kimi Delta Attention~\citep{Zhang2025KimiLA} further develops this direction with finer-grained gating in a hybrid architecture. Mixture-of-Memories~\citep{du2026mom} instead increases effective memory capacity by routing tokens across multiple independent memory states.
These approaches improve different aspects of bounded-state memory, including retention, updating, and capacity. \fracssm{} instead derives the temporal decay law by introducing a heavy-tailed fractional memory kernel, making it largely complementary to the update- and capacity-oriented mechanisms mentioned above.

\paragraph{Fractional dynamics in machine learning.} 
Fractional theory has found meaningful use in several areas of machine learning. 
Recent works have explored neural fractional differential equation models~\citep{Coelho2024NeuralFD} and scalable training methods for them~\citep{kang2025efficient}.
The work most closely related to ours is FADE~\citep{kang2025fade}, a fractional-attention differential-equation framework that solves a history-dependent neural integral equation over sampled past states using an iterative solver. Unlike \fracssm{}, FADE does not provide a fixed-size recurrent state or a decoder language-model architecture.
Fractional-order spiking neural networks apply the fractional formalism to enrich artificial neuron dynamics~\citep{ge2026fractionalorder}. Fractional stochastic dynamics has been used in generative modeling, including fractional diffusion~\citep{nobis2024generative} and a protein-generation extension~\citep{liang2025protgfdm}.
Our work instead approximates the fractional memory kernel with a finite SoE bank, giving a fixed-size recurrent state compatible with efficient scan computation and autoregressive decoding.

\section{Conclusion}

We introduced \fracssm{}, a selective SSM architecture that turns fractional long-memory dynamics into an efficient finite-state recurrent layer. Experiments show that \fracssm{} improves long-context performance over strong SSM baselines while remaining competitive on short-context benchmarks. Future work could combine fractional long-memory transitions with delta-rule style update mechanisms to obtain both stronger long-range inductive bias and more precise associative retrieval.

\bibliographystyle{plainnat}  
\bibliography{custom}


\appendix
\clearpage

\section{Discussion of Theorem~\ref{thm:diffusive-representation}}
\label{app:diffusive-proof}

The statement of this theorem is a standard, yet non-trivial, result in complex analysis. The exact reference can be found in the comprehensive book of \cite[Appendix E]{Mainardi2010}. Specifically, the construction of densities in diffusive representation is given there. First, for
$0<\alpha<1$ and $\lambda > 0$ : 
\begin{equation*}
E_{\alpha}(-\lambda t^{\alpha})
=
\int_{0}^{\infty} e^{-rt}\,K_{\alpha}(r;\lambda)\,dr,
\qquad t\ge 0,
\end{equation*}
where
\begin{equation*}
K_{\alpha}(r;\lambda)
=
\frac{1}{\pi}\,
\frac{\lambda\, r^{\alpha-1}\sin(\pi\alpha)}
{r^{2\alpha}+2\lambda r^{\alpha}\cos(\pi\alpha)+\lambda^{2}},
\qquad r>0.
\end{equation*}
Also for  $0<\alpha\le \beta<1$ and $\lambda > 0$:
\begin{equation*}
t^{\beta-1}E_{\alpha,\beta}(-\lambda t^{\alpha})
=
\int_{0}^{\infty} e^{-rt}\,K_{\alpha,\beta}(r;\lambda)\,dr,
\qquad t > 0,
\end{equation*}
where
\begin{equation*}
K_{\alpha,\beta}(r;\lambda)
=
\frac{1}{\pi}\,
\frac{
\lambda \sin\!\bigl(\pi(\beta-\alpha)\bigr)
+ r^{\alpha}\sin(\pi\beta)
}
{r^{2\alpha}+2\lambda r^{\alpha}\cos(\pi\alpha)+\lambda^{2}}
\,r^{\alpha-\beta},
\qquad r>0.
\end{equation*}

In both cases, the spectral density is nonnegative.

To get the formulas \eqref{eq:Ea-diffusive}, we need to change the variables $\tau = 1/r$. 
Then 
\begin{equation*}
    E_{\alpha}(-\lambda t^{\alpha}) =
\int_{0}^{\infty}
e^{-t/\tau}\, R_{\alpha,\lambda}(\tau)\, d\tau,
\end{equation*}
where 
\begin{equation}
\label{eq:r-density}
R_{\alpha, \lambda}(\tau) =     \frac{1}{\pi}\,
\frac{ \lambda \tau^{\alpha-1}\sin(\pi\alpha)}
{\lambda^2\tau^{2\alpha}+2\lambda \tau^{\alpha}\cos(\pi\alpha)+1},
\qquad \tau>0.
\end{equation}
Similarly, putting $\beta=\alpha$: 
\begin{equation*}
    g_{\alpha,\lambda}(t) := t^{\alpha-1}
E_{\alpha,\alpha}(-\lambda t^{\alpha}) = \int_{0}^{\infty}
e^{-t/\tau}\, H_{\alpha,\lambda}(\tau)\, d\tau ,
\end{equation*}
where 
\begin{equation}
\label{eq:h-density}
H_{\alpha,\lambda}(\tau) =     \frac{1}{\pi}\,
\frac{
\tau^{\alpha-2}\sin(\pi\alpha)
}
{1+2\lambda\tau^\alpha\cos(\pi\alpha)+\lambda^2\tau^{2\alpha}},
\qquad \tau>0.
\end{equation}

\begin{remark}
From the formulas it is clear that the two diffusive densities are related by the exact identity
\begin{equation}
\label{eq:h-r-relation}
H_{\alpha,\lambda}(\tau)
=
\frac{1}{\lambda\tau}\,R_{\alpha,\lambda}(\tau),
\qquad \tau>0.
\end{equation} 
\end{remark}

\begin{remark}
\label{rem:self-similarity}
It is important for the next steps to note the self-similarity relations in the \(\lambda\)-scaling. In particular, it follows from the explicit formulas for the densities that
\begin{align*}
R_{\alpha,\lambda}(\tau)
&=
\lambda^{1/\alpha}\,
R_{\alpha,1}\!\bigl(\lambda^{1/\alpha}\tau\bigr),
\\
H_{\alpha,\lambda}(\tau)
&=
\lambda^{2/\alpha-1}\,
H_{\alpha,1}\!\bigl(\lambda^{1/\alpha}\tau\bigr),
\end{align*}
so varying \(\lambda\) rescales the underlying timescale axis by the factor \(\lambda^{-1/\alpha}\).
\end{remark}

\section{Proof of Theorem~\ref{thm:soe-approx} }
\label{app:soe-proof}

\begin{proof}
We will use the diffusive representations \eqref{eq:Ea-diffusive} proved in
Appendix~\ref{app:diffusive-proof}:
\[
E_\alpha(-\lambda t^\alpha)
=
\int_0^\infty e^{-t/\tau}\,R_{\alpha,\lambda}(\tau)\,d\tau,
\qquad
g_{\alpha,\lambda}(t)
=
\int_0^\infty e^{-t/\tau}\,H_{\alpha,\lambda}(\tau)\,d\tau,
\]
where the densities are given in \eqref{eq:r-density} and \eqref{eq:h-density}.

Both densities are strictly positive on \((0,\infty)\). Setting $t=0$ for $E_{\alpha}(-\lambda t^{\alpha})$ in \eqref{eq:Ea-diffusive}, we get: 
\[
\int_0^\infty R_{\alpha,\lambda}(\tau)\,d\tau
=
E_\alpha(0)
=
1.
\]
The second equality is easily seen from the power series definition of the function $E_\alpha$ \eqref{eq:Ea}. In particular, \(R_{\alpha,\lambda}\in L^1(0,\infty)\).

Fix \(\varepsilon>0\), and set
$ \varepsilon_0:=\frac{\varepsilon}{2}\min(1,\lambda).
$

As we have shown, $R_{\alpha,\lambda}\in L^1(0,\infty)$, so we may choose
\(0<\tau_-<\tau_+<\infty\) such that
$$
\int_{(0,\tau_-)\cup(\tau_+,\infty)}
R_{\alpha,\lambda}(\tau)\,d\tau
<
\varepsilon_0.
\label{eq:tail-choice}
$$

For \(E_\alpha(-\lambda t^\alpha)\) kernel, 
using the inequality \(e^{-t/\tau}\le 1\):
\begin{equation}
\sup_{t\in[0,T]}
\int_{(0,\tau_-)\cup(\tau_+,\infty)}
e^{-t/\tau}R_{\alpha,\lambda}(\tau)\,d\tau
\le
\int_{(0,\tau_-)\cup(\tau_+,\infty)}
R_{\alpha,\lambda}(\tau)\,d\tau
<
\varepsilon_0
\le
\frac{\varepsilon}{2}.
\label{eq:hom-tail}
\end{equation}

For \(g_{\alpha,\lambda}(t)\) kernel, Fubini's theorem and the identity \eqref{eq:h-r-relation}
yield
\begin{align}
\int_0^T\!\!\int_{(0,\tau_-)\cup(\tau_+,\infty)}
e^{-t/\tau}H_{\alpha,\lambda}(\tau)\,d\tau\,dt
&=
\int_{(0,\tau_-)\cup(\tau_+,\infty)}
\Bigl(\int_0^T e^{-t/\tau}\,dt\Bigr)H_{\alpha,\lambda}(\tau)\,d\tau
\notag\\
&=
\frac{1}{\lambda}
\int_{(0,\tau_-)\cup(\tau_+,\infty)}
\bigl(1-e^{-T/\tau}\bigr)R_{\alpha,\lambda}(\tau)\,d\tau
\notag\\
&\le
\frac{1}{\lambda}
\int_{(0,\tau_-)\cup(\tau_+,\infty)}
R_{\alpha,\lambda}(\tau)\,d\tau
<
\frac{\varepsilon_0}{\lambda}
\le
\frac{\varepsilon}{2}.
\label{eq:forced-tail}
\end{align}

Now let
\[
A=\log\tau_-,
\qquad
B=\log\tau_+.
\]
For a sufficiently large \(M\in\mathbb N\) (to be justified below), let 
\[
h=\frac{B-A}{M},
\qquad
x_m=A+\Bigl(m-\tfrac12\Bigr)h,
\qquad m=1,\dots,M.
\]
Define
\[
\tau_m:=e^{x_m}
=
e^{A+h/2}e^{(m-1)h}
=
\tau_0 q^{m-1},
\]
where
\[
\tau_0:=e^{A+h/2}>0,
\qquad
q:=e^h>1.
\]
Thus \(\{\tau_m\}_{m=1}^M\) is a geometrically spaced bank.
With the change of variables \(\tau=e^x\), \(d\tau=e^x\,dx\), the truncated
integrals in \eqref{eq:Ea-diffusive}  become
\[
\int_{\tau_-}^{\tau_+} e^{-t/\tau}R_{\alpha,\lambda}(\tau)\,d\tau
=
\int_A^B F(t,x)\,dx,
\qquad
\int_{\tau_-}^{\tau_+} e^{-t/\tau}H_{\alpha,\lambda}(\tau)\,d\tau
=
\int_A^B G(t,x)\,dx,
\]
where
\[
F(t,x):=e^{-t e^{-x}}\,e^xR_{\alpha,\lambda}(e^x),
\qquad
G(t,x):=e^{-t e^{-x}}\,e^xH_{\alpha,\lambda}(e^x).
\]
The functions \(F\) and \(G\) are continuous on the compact rectangle
\([0,T]\times[A,B]\), hence uniformly continuous there (by Heine–Cantor Theorem).

Let
\[
\omega_F(\delta)
:=
\sup\bigl\{|F(t,x)-F(t,y)|:\ t\in[0,T],\ x,y\in[A,B],\ |x-y|\le \delta\bigr\},
\]
and define \(\omega_G\) analogously. Then
\(\omega_F(\delta)\to 0\) and \(\omega_G(\delta)\to 0\) as \(\delta \to 0^+ \).
The midpoint-rule estimate therefore gives
\begin{align}
\sup_{t\in[0,T]}
\left|
\int_A^B F(t,x)\,dx
-
h\sum_{m=1}^M F(t,x_m)
\right|
&\le
(B-A)\,\omega_F(h/2),
\label{eq:hom-interior}
\\
\sup_{t\in[0,T]}
\left|
\int_A^B G(t,x)\,dx
-
h\sum_{m=1}^M G(t,x_m)
\right|
&\le
(B-A)\,\omega_G(h/2).
\label{eq:forced-interior}
\end{align}
Choose \(M\) large enough so that
\[
(B-A)\,\omega_F(h/2)\le \frac{\varepsilon}{2},
\qquad
(B-A)\,\omega_G(h/2)\le \frac{\varepsilon}{2T}.
\]

By the definitions of \(F\) and \(G\), the midpoint approximations in
\eqref{eq:hom-interior} and \eqref{eq:forced-interior} take the form:
$$ 
h\sum_{m=1}^M F(t,x_m)
=
h\sum_{m=1}^M e^{-t e^{-x_m}}\,e^{x_m}R_{\alpha,\lambda}(e^{x_m})
=
h\sum_{m=1}^M e^{-t/\tau_m}\,\tau_m R_{\alpha,\lambda}(\tau_m) 
$$

$$
h\sum_{m=1}^M G(t,x_m)
=
h\sum_{m=1}^M e^{-t e^{-x_m}}\,e^{x_m}H_{\alpha,\lambda}(e^{x_m})
=
h\sum_{m=1}^M e^{-t/\tau_m}\,\tau_m H_{\alpha,\lambda}(\tau_m).
$$ 

This motivates the definitions for the coefficients:
\begin{equation}
c_m(\alpha,\lambda):=
h\,\tau_m\,R_{\alpha,\lambda}(\tau_m),
\qquad
d_m(\alpha,\lambda):=
h\,\tau_m\,H_{\alpha,\lambda}(\tau_m),
\qquad
m=1,\dots,M.
\label{eq:cmdm-def}
\end{equation}

Since \(R_{\alpha,\lambda}(\tau)\) and \(H_{\alpha,\lambda}(\tau)\) are positive
for all \(\tau>0\), we have \(c_m(\alpha,\lambda)>0\) and
\(d_m(\alpha,\lambda)>0\). 
Moreover,
\[
\frac{d_m(\alpha,\lambda)}{c_m(\alpha,\lambda)}
=
\frac{H_{\alpha,\lambda}(\tau_m)}{R_{\alpha,\lambda}(\tau_m)}
=
\frac{1}{\lambda\tau_m},
\qquad m=1,\dots,M.
\]

Combining \eqref{eq:hom-tail} and \eqref{eq:hom-interior}, we obtain
\[
\sup_{t\in[0,T]}
\left|
E_\alpha(-\lambda t^\alpha)
-
\sum_{m=1}^M c_m(\alpha,\lambda)e^{-t/\tau_m}
\right|
\le
\frac{\varepsilon}{2}+\frac{\varepsilon}{2}
=
\varepsilon.
\]
Likewise, by \eqref{eq:forced-tail}, \eqref{eq:forced-interior}, and the bound
\[
\int_0^T |f(t)|\,dt \le T\sup_{t\in[0,T]}|f(t)|,
\]
we get
\[
\int_0^T
\left|
g_{\alpha,\lambda}(t)
-
\sum_{m=1}^M d_m(\alpha,\lambda)e^{-t/\tau_m}
\right|dt
\le
\frac{\varepsilon}{2}
+
T\cdot\frac{\varepsilon}{2T}
=
\varepsilon.
\]

Finally, from the explicit formula for \(R_{\alpha,\lambda}\),
\[
R_{\alpha,\lambda}(\tau)
=
\frac{\sin(\pi\alpha)}{\pi\lambda}\,
\tau^{-\alpha-1}\bigl(1+O(\tau^{-\alpha})\bigr),
\qquad \tau\to\infty.
\]
Since \(h=\log q\), it follows from \eqref{eq:cmdm-def} that
\[
c_m(\alpha,\lambda)
=
(\log q)\,\tau_m R_{\alpha,\lambda}(\tau_m)
=
\frac{(\log q)\sin(\pi\alpha)}{\pi\lambda}\,
\tau_m^{-\alpha}\bigl(1+O(\tau_m^{-\alpha})\bigr)
\]
for the large-timescale part of the geometric bank. This completes the proof.
\end{proof}

\begin{remark}
\label{rem:norm-choice}
We use the uniform norm to approximate \(E_\alpha(-\lambda t^\alpha)\), while the approximation of
\(g_{\alpha,\lambda}(t)=t^{\alpha-1}E_{\alpha,\alpha}(-\lambda t^\alpha)\)
is stated in the \(L^1\)-norm. 
This asymmetry is forced by the structure of
the kernels themselves: for \(0<\alpha<1\), \(g_{\alpha,\lambda}\) has an
integrable singularity at \(t=0^+\), whereas any finite sum of exponentials
is bounded there, so no uniform pointwise approximation can hold near zero.
However, \(L^1\)-bound is 
sufficient for our main goal, the proof of Proposition~\ref{prop:soe-realization}.
\end{remark}

\begin{remark}
\label{rem:M-small}
Although Theorem~\ref{thm:soe-approx} is an existence result, the numerical-analysis
literature on sum-of-exponentials approximations for fractional kernels~\citep{JiangZhangZhang2017,Chaudhary2025} shows that
the required number of modes grows only poly-logarithmically with the target
accuracy and time horizon. Our kernels \(E_\alpha(-\lambda t^\alpha)\) and
\(g_{\alpha,\lambda}\) exhibit the same multi-timescale power-law structure, so
these works support the expectation that only a modest number of modes is needed
in our setting as well. In practice, this means that a bank of only a few tens
of modes can already cover many orders of magnitude of timescale at useful accuracy. See Appendix~\ref{app:ablation} for the practical ablation on the number of modes $M$.
\end{remark}

\section{Proof of Proposition~\ref{prop:soe-realization}}
\label{app:soe-realization-proof}

\begin{proof}
For each mode $m = 1, \dots, M$, the ODE
\[
\dot s_m(t)
=
-\frac{1}{\tau_m}s_m(t)
+
a_m(\alpha,\lambda)\,u(t),
\qquad
s_m(0)=h(0),
\]
has the explicit solution
\begin{equation}
s_m(t)
=
e^{-t/\tau_m}h(0)
+
a_m(\alpha,\lambda)\int_0^t e^{-(t-\xi)/\tau_m}u(\xi)\,d\xi.
\label{eq:app-bank-mode-solution}
\end{equation}
Therefore, the readout
\[
\tilde h_M(t)=\sum_{m=1}^M c_m(\alpha,\lambda)s_m(t)
\]
takes the form
\begin{align}
\tilde h_M(t)
&=
h(0)\sum_{m=1}^M c_m(\alpha,\lambda)e^{-t/\tau_m}
\notag\\
&\quad+
\int_0^t
\sum_{m=1}^M
c_m(\alpha,\lambda)a_m(\alpha,\lambda)e^{-(t-\xi)/\tau_m}u(\xi)\,d\xi.
\label{eq:app-bank-readout-expanded}
\end{align}
With the choice
\[
a_m(\alpha,\lambda)=\frac{d_m(\alpha,\lambda)}{c_m(\alpha,\lambda)},
\qquad m=1,\dots,M,
\]
this becomes
\begin{equation}
\tilde h_M(t)
=
h(0)\sum_{m=1}^M c_m(\alpha,\lambda)e^{-t/\tau_m}
+
\int_0^t
\sum_{m=1}^M
d_m(\alpha,\lambda)e^{-(t-\xi)/\tau_m}u(\xi)\,d\xi.
\label{eq:app-shared-bank-solution}
\end{equation}

On the other hand, the full solution $h$ for the fractional differential equation~\eqref{eq:fractional-relaxation}
is given by \eqref{eq:frac-full-solution}:

\begin{equation*}
h(t)
=
E_\alpha(-\lambda t^\alpha)\,h(0)
+
\int_0^t g_{\alpha,\lambda}(t-\xi)u(\xi)\,d\xi.
\end{equation*}

Subtracting \eqref{eq:app-shared-bank-solution} from
\eqref{eq:frac-full-solution}, we obtain
\begin{align}
|h(t)-\tilde h_M(t)|
&\le
|h(0)|
\left|
E_\alpha(-\lambda t^\alpha)
-
\sum_{m=1}^M c_m(\alpha,\lambda)e^{-t/\tau_m}
\right|
\notag\\
&\quad+
\left|
\int_0^t
\Bigl(
g_{\alpha,\lambda}(t-\xi)
-
\sum_{m=1}^M d_m(\alpha,\lambda)e^{-(t-\xi)/\tau_m}
\Bigr)u(\xi)\,d\xi
\right|.
\label{eq:app-error-split}
\end{align}

By \eqref{eq:shared-soe-homogeneous}, the first term is bounded by $ |h(0)|\,\varepsilon$.
For the second term, using \(\|u\|_{L^\infty([0,T])}<\infty\), the change of variables
\(r=t-\xi\), and then \eqref{eq:shared-soe-forced}, we get

\begin{align*}
&
\left|
\int_0^t
\Bigl(
g_{\alpha,\lambda}(t-\xi)
-
\sum_{m=1}^M d_m(\alpha,\lambda)e^{-(t-\xi)/\tau_m}
\Bigr)u(\xi)\,d\xi
\right|
\notag\\
&\le
\|u\|_{L^\infty([0,T])}
\int_0^t
\left|
g_{\alpha,\lambda}(t-\xi)
-
\sum_{m=1}^M d_m(\alpha,\lambda)e^{-(t-\xi)/\tau_m}
\right|d\xi
\notag\\
&=
\|u\|_{L^\infty([0,T])}
\int_0^t
\left|
g_{\alpha,\lambda}(r)
-
\sum_{m=1}^M d_m(\alpha,\lambda)e^{-r/\tau_m}
\right|dr
\notag\\
&\le
\|u\|_{L^\infty([0,T])}
\int_0^T
\left|
g_{\alpha,\lambda}(r)
-
\sum_{m=1}^M d_m(\alpha,\lambda)e^{-r/\tau_m}
\right|dr
\notag\\
&\le
\|u\|_{L^\infty([0,T])}\,\varepsilon.
\end{align*}

Substituting these two bounds into \eqref{eq:app-error-split} and taking the
supremum over \(t\in[0,T]\), we conclude that
\[
\sup_{t\in[0,T]}
|h(t)-\tilde h_M(t)|
\le
\varepsilon\Bigl(|h(0)|+\|u\|_{L^\infty([0,T])}\Bigr),
\]
which is exactly \eqref{eq:soe-realization-error}. This completes the proof.
\end{proof}

\section{Frac State Transition Algorithm }
\label{app:frac_transition_algo}
For compactness, Algorithm~\ref{alg:frac-state-transition} is written in vectorized form over heads \(H\) with head dimension \(d_h\). Equivalently, the same update is applied independently to each head with its own controls \(\Delta_t\), \(\alpha_t\), and \(\lambda_t\).

In our implementation, the maps \(g_{\mathrm{read}}\) and \(g_{\mathrm{write}}\) are parameterized as gated linear projections of the form
\[
g_{\mathrm{read}}(u)=\sigma(p_{\mathrm{read}})\,W_{\mathrm{read}}u,
\qquad
g_{\mathrm{write}}(u)=\sigma(p_{\mathrm{write}})\,W_{\mathrm{write}}u,
\]
where \(p_{\mathrm{read}}\) and \(p_{\mathrm{write}}\) are learnable scalar parameters and \(W_{\mathrm{read}},W_{\mathrm{write}}\) are learnable matrices.

\begin{algorithm}[h]
\caption{Frac state transition for one token $t$ (single-step recurrent form)}
\label{alg:frac-state-transition}
\begin{algorithmic}[1]
\Require
input $u_t \in \mathbb{R}^{H \times d_h}$,
step sizes $\Delta_t \in \mathbb{R}_{>0}^{H}$,
fractional controls $\alpha_t \in (0,1)^{H}$ and $\lambda_t \in \mathbb{R}_{>0}^{H}$,
previous mode state $S_{t-1} \in \mathbb{R}^{H \times M \times d_h}$,
base timescale bank tensor $\tau \in \mathbb{R}_{>0}^{M}$
\Statex content-dependent linear maps $g_{\mathrm{read}}, g_{\mathrm{write}} : \mathbb{R}^{H \times d_h} \to \mathbb{R}^{H \times M}$
\Statex
\Comment{Fractional specialization of the ZOH coefficients}
\State $\tilde{\tau}_t \gets \tau \oslash \lambda_t^{1/\alpha_t}$
\State $\tilde{\rho}_t \gets \exp(-\Delta_t \oslash \tilde{\tau}_t)$
\State $\tilde{\beta}_t \gets 1 - \tilde{\rho}_t$

\Statex
\Comment{Read and write weights over the mode bank}
\State $c_t \gets \mathrm{softmax}_m\!\left(-\alpha_t \log \tau + g_{\mathrm{read}}(u_t)\right)$
\State $b_t \gets \mathrm{softmax}_m\!\left(-\alpha_t \log \tau + g_{\mathrm{write}}(u_t)\right)$
\State $\kappa_t \gets \tilde{\beta}_t \odot b_t$

\Statex
\Comment{Selective recurrent update and readout}
\State $S_{t} \gets \tilde{\rho}_t \odot S_{t-1} + \kappa_t \odot u_t$
\State $h_t \gets \sum_{m=1}^M c_{t,:,m}\odot S_{t,:,m,:}$

\State \Return $h_t,\; S_{t}$
\end{algorithmic}
\end{algorithm}

\begin{remark}
The precise ZOH discretization, as we derive in \eqref{eq:rho-beta-def-frac}, gives $\tilde{\beta}_t = (1 - \tilde{\rho}_t ) / \lambda_t$. 
In Algorithm~\ref{alg:frac-state-transition}, however, we use $\tilde{\beta}_t = 1 - \tilde{\rho}_t$ (Step~3). We made this modeling choice to decouple timescale control from the write amplitude, which we found improves training stability. When $\lambda$ is input-independent, the two forms are equivalent up to a constant rescaling of the state and readout. All our experiments in this paper use $\tilde{\beta}_{t,m} = 1-\tilde{\rho}_{t,m}$, as in Algorithm~\ref{alg:frac-state-transition}.
\end{remark}

\begin{remark}
    The ZOH derivation in Section~\ref{subsec:frac-transition} is exact for a
fixed fractional mode bank with frozen values of \(\alpha\), \(\lambda\), and
\(\Delta\). The selective \fracssm{} layer applies the same transition
parameterization token-wise, with \(\alpha_t\), \(\lambda_t\), and
\(\Delta_t\) frozen within each interval, while the learned weights
\(b_{t,m}\) and \(c_{t,m}\) provide content-dependent write and read routing
over the resulting mode bank.  Thus, the fractional theory specifies the
per-token memory geometry and ZOH transition form, while selectivity adds
content-dependent routing over this fractional transition. 
Accordingly, the full selective layer can be viewed as an architectural generalization of the fixed-system fractional construction, which serves as a principled inductive bias for heavy-tailed memory.
\end{remark}

\section{Experimental Setting}
\label{app:experimental-setting}

\subsection{Heavy-Tail Synthetic Probing}
\label{app:heavy-tail-synth}
We use a controlled synthetic task to isolate long-range extrapolation under a heavy-tail aggregation law. Each sequence contains sparse $\pm 1$ events embedded in a background token, with inter-event gaps drawn from a truncated Zipf distribution. The label is the sign of a fixed power-law weighted sum over preceding events,
\[
y = \sum_k \frac{v_k}{(L-k)^{\gamma}},
\]

where the sum runs over all event positions $k$, each event carries value $v_k \in \{-1,+1\}$, and $\gamma=0.1$ is chosen so that the target depends on a strongly heavy-tailed accumulation over distance. We train all models on sequences of length $512$ and evaluate zero-shot extrapolation at $512, 2$K, $8$K, $16$K, $32$K, $64$K, and $128$K.
All models use a single sequence-mixing layer and, as shown in Table~\ref{tab:heavy_tail_model_sizes}, are approximately parameter-matched at about $0.2$M parameters. 
For \mambathree{}, we use the stable \textsc{SISO} variant, since the \textsc{MIMO} implementation in TileLang does not support hidden sizes small enough to match the parameter counts of the other models. Table~\ref{tab:heavy_tail_app} reports mean accuracy and standard deviation across runs for each model at the tested sequence lengths.

\begin{table}[h]
\centering
\small
\setlength{\tabcolsep}{7pt}
\begin{tabular}{lcccc}
\toprule
Model & Hidden & Heads & Head dim & Params \\
\midrule
Attention        & 128 & 8 & 16 & 199.3K \\
\mambatwo{} & 160 & 8 & 40 & 200.7K \\
\mambathree{} & 160 & 8 & 40 & 201.3K \\
\textsc{GDN}     & 96  & 8 & 32 & 203.3K \\
\textsc{Frac}    & 176 & 8 & 44 & 195.2K \\
\bottomrule
\end{tabular}
\caption{Model configurations for the heavy-tail synthetic benchmarking.}
\label{tab:heavy_tail_model_sizes}
\end{table}

\begin{table}[h]
\centering
\small
\setlength{\tabcolsep}{4.0pt}
\begin{tabular}{lccccccc}
\toprule
Model & 512 & 2K & 8K & 16K & 32K & 64K & 128K \\
\midrule
Attention
& $97.6{\scriptscriptstyle \pm 0.8}$
& $65.9{\scriptscriptstyle \pm 3.7}$
& $50.5{\scriptscriptstyle \pm 3.3}$
& $50.1{\scriptscriptstyle \pm 3.0}$
& $50.9{\scriptscriptstyle \pm 3.2}$
& $48.6{\scriptscriptstyle \pm 1.7}$
& $50.0{\scriptscriptstyle \pm 3.7}$ \\
\mambatwo{}
& $82.7{\scriptscriptstyle \pm 2.9}$
& $74.0{\scriptscriptstyle \pm 3.6}$
& $60.9{\scriptscriptstyle \pm 4.5}$
& $59.4{\scriptscriptstyle \pm 5.4}$
& $56.0{\scriptscriptstyle \pm 3.5}$
& $57.9{\scriptscriptstyle \pm 3.2}$
& $51.3{\scriptscriptstyle \pm 4.3}$ \\
\textsc{\mambathree{}}
& $\mathbf{98.6{\scriptscriptstyle \pm 0.7}}$
& $90.5{\scriptscriptstyle \pm 2.7}$
& $71.3{\scriptscriptstyle \pm 7.5}$
& $67.0{\scriptscriptstyle \pm 5.0}$
& $62.6{\scriptscriptstyle \pm 4.3}$
& $63.2{\scriptscriptstyle \pm 4.4}$
& $55.6{\scriptscriptstyle \pm 3.8}$ \\
\gdn{}
& $98.5{\scriptscriptstyle \pm 0.6}$
& $\mathbf{92.5{\scriptscriptstyle \pm 2.8}}$
& $78.7{\scriptscriptstyle \pm 3.4}$
& $73.7{\scriptscriptstyle \pm 3.5}$
& $68.4{\scriptscriptstyle \pm 2.0}$
& $65.8{\scriptscriptstyle \pm 2.5}$
& $58.4{\scriptscriptstyle \pm 3.8}$ \\
\fracssm{}
& $96.5{\scriptscriptstyle \pm 2.5}$
& $92.3{\scriptscriptstyle \pm 1.1}$
& $\mathbf{85.6{\scriptscriptstyle \pm 3.0}}$
& $\mathbf{80.4{\scriptscriptstyle \pm 1.4}}$
& $\mathbf{75.2{\scriptscriptstyle \pm 0.4}}$
& $\mathbf{70.6{\scriptscriptstyle \pm 2.0}}$
& $\mathbf{63.8{\scriptscriptstyle \pm 1.9}}$ \\
\bottomrule
\end{tabular}
\caption{Accuracy and standard deviation on  heavy-tail synthetic benchmark across $512$-$128$K sequence lengths on 10 runs. Bold indicates the best score at each sequence length.}
\label{tab:heavy_tail_app}
\end{table}

\subsection{MadLab Synthetic Suite}
\label{app:madlab}

\paragraph{Setup} MADLab~\citep{poli24_madlab} is a synthetic benchmark suite designed to probe token-level sequence manipulation and recall mechanisms. We consider the six standard tasks. \emph{Compression} (Comp.) tests whether a model can retain and compress information from the input sequence. \emph{In-context recall} (ICR) evaluates associative retrieval of values from keys presented in context. \emph{Noisy in-context recall} (N-ICR) adds distractor tokens to the same basic retrieval problem of ICR. \emph{Fuzzy in-context recall} (F-ICR) makes the association less direct by requiring recall from longer or less trivially matched motifs. \emph{Selective copying} (SC) measures the ability to copy only the relevant marked tokens from a sequence, and \emph{memorization} (Mem.) tests direct token-level memorization of the training distribution.

\paragraph{Model Configurations} 
We report the performance of four linear models: our \fracssm{}, \mambatwo{}, \gdn{}, and \mambathree{}. As in Appendix~\ref{app:heavy-tail-synth}, we use only the \mambathree{\small{\textit{-SISO}}} variant, since the \textsc{MIMO} implementation does not support sufficiently small hidden dimensions. All models use the same 4-layer architecture: two sequence-mixing layers, each followed by a SwiGLU channel-mixing layer.
All models use fixed hidden size \(d=128\) following the original protocol in the MADLab paper, and causal convolution size \(d_{\mathrm{conv}}=4\). We matched the parameter count to make all models have approximately 0.5M parameters. \mambatwo{} uses expansion \(2.0\), while all other models use expansion \(1.5\). All models use 8 heads, except for \fracssm{} which uses 6 heads. For \fracssm{}-specific hyperparameters we use the same setting as for language modeling experiments, the justification for which can be found in the next subsection.  

\paragraph{Training protocol}
Models are trained from scratch with AdamW, batch size \(128\) for \(200\) epochs, cosine learning-rate schedule, minimum learning rate \(10^{-6}\), bfloat16 precision, and \(1280\) test examples. 
We sweep learning rates \(\{1e^{-4},5e^{-4},1e^{-3}\}\) and weight decay values \(\{0,0.1\}\). 
For each task and model, we select the best hyperparameter setting by the mean official MADLab score across five seeds, and report the mean and standard deviation over the five runs for the selected setting.
We use the official MADLab baseline configurations for all task settings, including sequence length, vocabulary size, motif size, noise level, copy length, and multi-query flags. We do not report performance for N-ICR and ICR in the main paper as all models saturated on these two tasks. 

\subsection{Language Modeling}
\label{app:Language Modeling}

\paragraph{Pretraining Data} We leverage the deduplicated FineWeb-Edu (220B tokens total) subsets of the  SmolLM-Corpus~\citep{benallal2024smollmcorpus} as pretraining data in all of our experiments. FineWeb-Edu is a deduplicated, high-quality subset of educational web data filtered from the FineWeb-v1 collection~\cite{penedo2024the}. For our experiments, we randomly sample a 100B-token subset from deduplicated FineWeb-Edu, measured after tokenization with the Llama-2~\citep{arxiv23_llama2} tokenizer  of a vocabulary size of 32k tokens.

\paragraph{Model Configurations}

In our main experiment, all models are matched to roughly the same parameter count of 1.3B. For the baseline linear models, we adopt the default configuration of the 1B-parameter models from their respective publicly available repositories. We use the Qwen3~\cite{yang2025qwen3} architecture as the backbone for the Transformer model, as it achieves state-of-the-art performance among open-weight pure self-attention dense models. All models use a hidden size of 2048, an expansion factor for the intermediate state of 2.0\footnote{Except for \gdn{}, where we set it to 3.0 as in their default configuration.}, 16 heads, while the number of layers is determined by each model’s original configuration and adjusted so that the total parameter count is around 1.3B. More precisely, we use 48 layers for \mambatwo{}, \mambathree{\small{\textit{-SISO}}}, and \fracssm{}; 44 layers for \mambathree{\small{\textit{-MIMO}}}; and 22 and 26 layers for \gdn{} and the Transformer, respectively. 

For \fracssm{}, we apply min-max clipping of $\Delta \in [10^{-4}, 1.0]$, following Mamba~\citep{mamba2,lahoti2026mamba3}, and also clip our $\lambda \in [0.25, 4.0]$ for numerical stability. 
We set the number of modes to \(M=16\). This choice follows the scale suggested by Remark~\ref{rem:M-small}: a bank of a few tens of geometrically spaced modes is expected to cover many orders of magnitude of timescale at useful accuracy. For implementation, we use a power-of-two mode count to improve hardware utilization in our Triton kernels, making \(M=16\) the smallest power-of-two value in this regime. In the ablation study of Appendix~\ref{app:ablation}, \(M=8\) degrades performance, while \(M=32\) gives only the same or a small improvement at a higher computational cost.
For the timescale grid $\tau$, we use geometrically spaced base timescales between \(\tau_{\min}=1\) and \(\tau_{\max}=2^{17}\).
The upper endpoint is chosen to exceed the longest extrapolation length considered in our experiments, \(64\mathrm{k}=2^{16}\), so the slowest base mode remains active over the full evaluation horizon. 
The lower endpoint gives the bank access to local recurrent memory, while the logarithmic spacing allocates the remaining modes across intermediate scales. It is worth noting that linear baseline models like Mamba and \gdn{} have no explicit maximum-timescale parameter analogous to  \(\tau_{\max}\), while the Transformer can directly attend to the full context.

Importantly, introducing this finite mode bank does not give \fracssm{} a larger recurrent
state than the baseline linear models. To quantify this, we compare the recurrent SSM state tensors used during autoregressive
decoding, excluding short-convolution buffers and architecture-specific
auxiliary caches. In our configurations, \fracssm{}, \mambatwo{}, and
\mambathree{} use \(d_{\mathrm{inner}}=4096\) and \(H=16\), giving
\(d_h=d_{\mathrm{inner}}/H=256\). With \(M=16\) modes, \fracssm{} maintains
\(S_t^{\mathrm{FRAC}}\in\mathbb{R}^{H\times M\times d_h}\), corresponding to
\(16\times16\times256=65{,}536\) recurrent scalars per layer. The principal
SSM state in \mambatwo{} and \mambathree{} has shape
\(S_t^{\mathrm{Mamba}}\in\mathbb{R}^{H\times d_h\times N}\); with \(N=128\),
this corresponds to \(16\times256\times128=524{,}288\) scalars per layer.
Similarly, \gdn{} maintains a matrix-valued state
\(S_t^{\mathrm{GDN}}\in\mathbb{R}^{H\times d_k\times d_v}\); with \(d_k=96\)
and \(d_v=192\), this gives \(16\times96\times192=294{,}912\) scalars per
layer. Thus, Mamba2\&3 and \gdn{} maintain respectively
\(8\times\) and \(4.5\times\) as many recurrent-state scalars as \fracssm{}. We also note that recurrent-state size is different from trainable parameter count, and each architecture has to allocate weights to other parts of the token mixer. For example, \fracssm{} includes the projections producing
\(\Delta_t\), \(\alpha_t\), and \(\lambda_t\), together with its read/write
projections, whereas \gdn{} allocates parameters to its query, key, value,
gating, and output projections.

\paragraph{Implementation Details}

Each model is pretrained on 2 compute nodes, each equipped with 8 \textit{modern parallel compute cards} with 80GB of memory per card. We use the Hugging Face Transformers library~\citep{wolf2020transformers} as our pretraining framework. For baselines, we use the official Mamba2\&3\footnote{\url{https://github.com/state-spaces/mamba}} and GDN\footnote{\url{https://github.com/fla-org/flash-linear-attention/tree/main}} implementations, and the Qwen3 implementation in the Transformers library for the Transformer baseline. For all models, we use the AdamW~\cite{loshchilov2017decoupled} optimizer with a learning rate decay setting the initial learning rate to 6e-4 with 10\%  warm-up steps, with a global batch size of $1$M tokens. We use a weight decay of 0.1 for all models, following prior work. The selection of the learning rate was based on preliminary training runs, where it produced stable loss curves across all models. A higher learning rate, such as 1e-3, led to early training divergence, whereas learning rates substantially below 6e-4 produced worse loss curves for most models. To accelerate pretraining, we use Fully Sharded Data Parallel (FSDP)~\cite{zhao2023pytorch} and mixed-precision training~\citep{fp16}. Pretraining each model was completed in roughly 4 to 5 days (see Table~\ref{tab:app_latency}).
 
\paragraph{Evaluation Protocol} 

We mostly follow the evaluation protocol from Mamba2\&3 and GDN by conducting comprehensive evaluations of models we pretrain from scratch on 4 benchmark suites\footnote{Except for Real-World Recall-Retrieval tasks, all benchmarks are implemented through the LM Eval Harness library~\cite{eval-harness}.}:

\begin{itemize}
    \item \textbf{LM Harness} which includes perplexity evaluation on Wikitext (Wiki.)~\citep{merity2016pointer} and LAMBADA (LMB.)~\citep{paperno_lambada_2016}, as well as several commonsense reasoning tasks: PIQA~\citep{bisk2020piqa}, HellaSwag (Hella)~\citep{zellers2019hellaswag}, WinoGrande (Wino)~\citep{sakaguchi2021winogrande}, ARC-Easy (ARC-e) and ARC-Challenge (ARC-c)~\citep{allenai:arc}, OpenBookQA (OBQA)~\citep{mihaylov2018can}, and BoolQ~\citep{clark2019boolq}. 
    
    \item \textbf{NIAH} From the \textsc{RULER}~\cite{hsieh2024ruler} benchmark, we include three single \textit{needle-in-a-haystack} tasks: \textit{S-NIAH-1} (pass-key retrieval from repetitive filler text), \textit{S-NIAH-2} (retrieving a target number embedded in natural-text distractors), and \textit{S-NIAH-3} (retrieving a target UUID embedded in natural-text distractors). 

    \item \textbf{LongBench} We consider 14 tasks from the LongBench (v1) benchmark~\citep{bai2024longbench} and follow the authors’ categorizations: Single-document QA (NarrativeQA~\citep{kocisky-etal-2018-narrativeqa}, MultiFieldQA-en, Qasper~\citep{dasigi2021qasper}), Multi-document QA (HotpotQA~\citep{yang2018hotpotqa}, 2WikiMQA~\citep{ho2020constructing}, MuSiQue~\citep{trivedi2022musique}), Summarization (GovReport~\citep{huang_efficient_2021}, MultiNews~\citep{fabbri2019multinews}, QMSum~\citep{zhong2021qmsum}), Few-shot learning (TREC~\citep{li2002learning}, TriviaQA~\citep{joshi2017triviaqa}, SAMSum~\citep{gliwa2019samsum}), and Code Completion (LCC~\citep{guo2023longcoder}, RepoBench-P~\citep{liu2023repobench}).

    \item \textbf{Real-World Recall-Retrieval} We measure performance on real-world retrieval tasks, including: SWDE  \citep{lockard_openceres_2019}, SQuAD (SQD) \citep{rajpurkar_know_2018}, FDA \citep{wu_how_2021}, TriviaQA (TQA) \citep{joshi2017triviaqa}, NQ \citep{47761}, and Drop \citep{dua2019drop}. Following prior work~\citep{mamba2,lahoti2026mamba3,yang2025gdn}, we evaluate using the cloze-completion prompt format of~\cite{arora-2024-jrt}, with all sequences truncated up to 2k tokens.
   
\end{itemize}

\begin{table}[!tph]
\centering
\small
\begin{tabular}{llrcc}
\toprule
Task & Acr. & \# Ex. & Metric & Length \\
\midrule
NarrativeQA       & NQA & 200 & F1        & $18405 \pm 8976$ \\
Qasper            & QQA & 200 & F1        & $3619 \pm 1880$ \\
MultiFieldQA-en   & MFQ & 150 & F1        & $4559 \pm 2473$ \\
HotpotQA          & HQA & 200 & F1        & $9149 \pm 2849$ \\
2WikiMQA          & 2WM & 200 & F1        & $4885 \pm 2541$ \\
MuSiQue           & Mus & 200 & F1        & $11018 \pm 1560$ \\
GovReport         & GvR & 200 & ROUGE     & $8169 \pm 5239$ \\
QMSum             & QMS & 200 & ROUGE     & $10546 \pm 5043$ \\
MultiNews         & MNs & 200 & ROUGE     & $2114 \pm 1554$ \\
TREC              & TRC & 200 & Acc.      & $5176 \pm 2140$ \\
TriviaQA          & TQA & 200 & F1        & $8209 \pm 4029$ \\
SAMSum            & SSM & 200 & ROUGE     & $6258 \pm 3107$ \\
LCC               & LCC & 500 & Edit Sim. & $1235 \pm 1086$ \\
RepoBench-P       & RBP & 500 & Edit Sim. & $4206 \pm 2698$ \\
\bottomrule
\end{tabular}
\caption{Characteristics of the LongBench tasks, including acronym, number of examples, and evaluation metric. The last column reports the mean and standard deviation of example lengths, computed using the Llama-2 tokenizer used in our experiments. Eight out of 14 tasks have an average sequence length less than $8$K.}
\label{tab:longbench_stats}
\end{table}

\subsection{DNA Modeling}
\label{app:dna}

We follow the DNA language-modeling setup of HyenaDNA~\citep{nguyen2023hyenadna}\footnote{\url{https://github.com/HazyResearch/hyena-dna}} and Mamba~\citep{gu2024mamba}. We use the HG38 human reference genome data distributed through the HyenaDNA preprocessing pipeline, and train models with a character-level DNA tokenizer under the standard autoregressive next-token prediction objective. Given the training set, we truncate each example to a fixed sequence length. After training, we evaluate perplexity on the test set using the same truncation length as during training. We conduct this experiment for the following lengths \(\{1K, 4K, 8K, 16K, 32K, 64K\}\), with the resulting test perplexities shown in Figure~\ref{fig:dna}, comparing \fracssm{} against \gdn{} and \mambathree{\small{\textit{-SISO}}}.
All models use the same 8-layer decoder-only language-model core with hidden size \(256\), convolution size $4$ and eight heads. We approximately match model size to $7$M parameters by adjusting the model-specific expansion. More precisely, we use expansion factors of \(1.75\), \(2.0\), and \(1.5\) for \mambathree{}, \gdn{}, and \fracssm{}, respectively. Otherwise, we use the same model-specific hyperparameters as in the language-modeling experiments in Appendix~\ref{app:Language Modeling}. We sweep learning rates over the grid  $
\{1\mathrm{e}{-4}, 5\mathrm{e}{-4}, 1\mathrm{e}{-3}, 2\mathrm{e}{-3},
4\mathrm{e}{-3}, 8\mathrm{e}{-3}, 1\mathrm{e}{-2}, 2\mathrm{e}{-2}\}$ and select the best on the validation set. We use AdamW with weight decay $0.1$, and the same cosine learning-rate schedule.

\section{FracMixer Efficient Implementation}
\label{app:FracMixer Efficient Implementation}

In our implementation, the dominant components of \textsc{FracMixer} are implemented using custom Triton kernels. The inputs to the \textsc{Frac State Transition} block are first computed with standard matrix multiplications, followed by a fused Triton kernel that constructs the scan parameters, including the read/write mode weights and the $\Delta$, $\alpha$, and $\lambda$ factors. \textsc{Frac State Transition} scan is computed by a sequence of chunk-local Triton kernels for cumulative decays, within-chunk state construction, cross-chunk state passing, chunk-local read/write mixing, and the final scan output. In addition, the output skip connection is fused in that kernel. Although the fractional memory bank introduces multiple timescale modes, the computation is still organized around efficient chunked tensor operations.
Compared with \mambathree{\small{\textit{-MIMO}}}, FracMixer uses a simpler real-valued diagonal-mode update and does not employ complex rotations, trapezoidal two-tap state-inputs, or MIMO rank projections. However, \mambathree{\small{\textit{-MIMO}}} uses a lower-level CuTe/TileLang implementation, which is more optimized than our Triton, and a smaller rank $R\!=\!4$ compared to our $M\!=\!16$ modes.

\begin{table}[!thp]
\centering

\begin{tabular}{l|cccc|c}
\toprule
\bf Model & \bf Prefill$@1$K & \bf Prefill$@4$K & \bf Prefill$@16$K & \bf Decode & \bf Train$@4$K \\
\midrule
Transformer & \textbf{55,839} & \textbf{79,173} & 56,297 & \textbf{62}* & 218,184 \\
\mambatwo{} & 15,362 & 41,060 & 67,571 & \underline{34} & \bf 278,620 \\
\gdn{} & 10,455 & 35,481 & 58,384 & 20 & 215,357 \\
\mambathree{\small{\textit{-SISO}}} & \underline{19,517} & 44,998 & \underline{74,314} & 27 & 238,924 \\
\mambathree{\small{\textit{-MIMO}}} & 13,692 & 30,710 & 46,707 & 25 & 200,994 \\
\fracssm{} & 16,213 & \underline{49,946} & \textbf{77,490} & 33 & \underline{241,598} \\
\bottomrule
\end{tabular}
\caption{Inference throughput in tokens/s (higher is better) for prefill at $1$K, $4$K, and $16$K context lengths, and single-step decoding throughput. All prefill measurements use batch size $1$. *Decode throughput for the Transformer is reported at $16$K, while the corresponding values at $1$K and $4$K are $75$ and $72$ tokens/s, respectively. The last column reports training throughput (also in tokens/s) following the experimental setup described in Appendix~\ref{app:Language Modeling}.}
\label{tab:app_latency}
\end{table}

Table~\ref{tab:app_latency} reports the inference prefill throughput (tokens/s) of the 1.3B models of Appendix~\ref{app:Language Modeling}.\footnote{Throughput is measured on a single card of the same \textit{modern compute accelerator} used for training, at sequence lengths of $\{1\mathrm{K},4\mathrm{K},16\mathrm{K}\}$ with batch size 1. All models are evaluated under identical conditions.
} We also report autoregressive decoding throughput, averaged over the generation of 64 tokens. It is worth noting that higher prefill throughput at longer contexts reflects improved GPU utilization and better amortization of fixed overheads. For instance, a single $16$K-token sequence contains more tokens to process than a $4$K-token sequence.

We observe that the self-attention Transformer achieves the highest throughput on short sequences ($1$--$4$K) as well as on single token decoding, mainly due to the highly optimized FlashAttention kernel. Among linear models, \fracssm{} and \mambathree{\small{\textit{-SISO}}} are the closest to the Transformer in the short-context regime. However, at $16$K, the quadratic complexity of self-attention becomes more pronounced, causing the Transformer to lag behind linear models. In this long-context regime, \fracssm{} achieves the highest throughput, in the same range as \mambathree{\small{\textit{-SISO}}} and higher than \mambathree{\small{\textit{-MIMO}}}. In the training setting of Appendix~\ref{app:Language Modeling}, throughput depends on both the forward and backward implementations. We observe that all models fall within a similar throughput range, with \mambatwo{} being the most efficient and slightly outperforming our \fracssm{}.

It is worth noting that the diagonal state transition can introduce a
numerical issue analogous to the parallel formulation of GLA~\citep{yang2024gla}:
expressing the decay between positions $t$ and $k$ as a ratio of cumulative
exponentials can be numerically unstable. We therefore compute the intra-chunk
mixing coefficient directly using cumulative log-decay differences:
\begin{equation*}
\Gamma_{t,k}
=
\sum_{m=1}^{M}
c_{t,m}\kappa_{k,m}
\exp\!\left(\ell_{t,m}-\ell_{k,m}\right),
\qquad t\geq k,
\end{equation*}
where
\(
\ell_{t,m}:=\sum_{j=t_{\mathrm{start}}}^{t}
\log \tilde{\rho}_{j,m}
\)
denotes the cumulative log-decay of mode $m$ from the beginning of the
current chunk to position $t$. This formulation computes the decay between
positions $k$ and $t$ directly in log space and is analogous to the stable
formulation used in GLA~\citep{yang2024gla}. GLA uses a second level of
chunking to retain tensor-core acceleration for its relatively large
reduction dimension ($d_k\geq64$). In \fracssm{}, the corresponding reduction
is over only $M=16$ modes, making it inexpensive to compute directly within
a fused Triton kernel without tensor cores. We therefore use a single level
of chunking, with parallel intra-chunk computation and chunk-level state
passing, without requiring the additional sub-chunking used by GLA.

\section{\fracssm{} Ablation}
\label{app:ablation}

We conduct a set of ablation experiments on the \fracssm{Mixer} block to validate our architectural design choices. To this end, we define a language modeling setup in which \fracssm{} models with 390M parameters are pretrained from scratch on a 30B-token ($4$K packed sequences) subset of the corpus used in the experiments in Appendix~\ref{app:Language Modeling}. Specifically, we use a model with hidden size 1024, intermediate expansion factor 2.0, and 36 layers. Otherwise, all remaining configuration details follow those of the \fracssm{} 1.3B model in Appendix~\ref{app:Language Modeling}. We validate these ablations by measuring long-document language modeling perplexity on several datasets with $16$K-token sequences\footnote{Training is performed at $4$K sequence length.}: ProofPile\footnote{\url{https://github.com/zhangir-azerbayev/proof-pile}}, PG19~\citep{rae_compressive_2020}, and GovReport~\citep{huang_efficient_2021}. This evaluation protocol is better suited for small-scale models and pretraining experimental settings~\citep{lu2025mamba,yelongmamba}.

\begin{wraptable}{r}{0.52\linewidth}
\vspace{-3mm}
\centering
\small
\setlength{\tabcolsep}{3pt}
\begin{tabular*}{\linewidth}{@{\extracolsep{\fill}}lccc@{}}
\toprule
\textbf{Model} 
& \textbf{ProofPile} 
& \textbf{PG19} 
& \textbf{GovReport} \\
& \textbf{ppl$\downarrow$} 
& \textbf{ppl$\downarrow$} 
& \textbf{ppl$\downarrow$} \\  
\midrule
\textbf{\fracssm{}} & \bf 47.2 & \bf 28.4 & 11.3 \\
\hspace{2mm} $\alpha=1$ & 61.3 & 36.2 & 19.7 \\ 
\hspace{2mm} \textit{w/o} write prior & 68.5 & 41.6 & 18.2 \\ 
\hspace{2mm} pure multiscale bank & 76.3 & 47.1 & 17.2 \\
\hspace{2mm} \textit{w/o} softmax R/W & 80.4 & 36.9 & 15.5 \\
\hspace{2mm} learned timescales & 54.1 & 34.3 & 13.5 \\
\hspace{2mm} \textit{w/o} $D$ & 70.9 & 46.4 & 23.6 \\
\hspace{2mm} $M\!=\!8$ & 51.4 & 33.9 & 14.5 \\
\hspace{2mm} $M\!=\!32$ & 47.9 & 29.1 & \bf 10.6 \\
\bottomrule
\end{tabular*}
\caption{Perplexity scores of 390M-parameter \fracssm{} models under ablations of \fracssm{} design choices.}
\label{tab:frac_ablation_ppl}
\vspace{-4mm}
\end{wraptable}
First, we observe that making $\alpha$ non-learnable (\(\alpha=1\)) or removing the prior term for write\footnote{Removing \(-\alpha_t \log \tau\) in line~5 of Algorithm~\ref{alg:frac-state-transition}.} (\textit{w/o} write prior) causes a drastic drop in language-modeling performance, validating our design choices. In addition, dropping the direct feedthrough term \(D\) (\textit{w/o} \(D\)) leads to worse performance, aligning with prior observations that \(D\) is an important component in modern SSM mixers~\citep{gu2024mamba,mamba2,lahoti2026mamba3,yang2025gdn}. Second, we observe that reducing the number of modes to \(M=8\) consistently degrades performance, whereas increasing it to \(M=32\) yields performance comparable to the baseline at the cost of 5\% more parameters and 7\% slower computation. We therefore use \(M=16\) in all experiments.

To further isolate which components stemming from the FDE motivation are responsible for the gains, we conduct three additional ablations. First, we remove all fractional inductive biases and train a pure multiscale memory bank baseline. Here, we fix $\alpha=1$, remove the read and write fractional priors, and use freely learned timescales initialized independently rather than on a geometric grid (\textit{pure multiscale bank}). Second, we remove the softmax-normalized read and write weights while keeping learnable $\alpha$ (\textit{w/o} softmax R/W). Third, we keep our \fracssm{} layer with learnable $\alpha$, priors on read and write, and softmax normalization, and only replace $\tau$ on a geometric grid with freely learned timescales initialized independently (\textit{learned timescales}). As shown in Table~\ref{tab:frac_ablation_ppl}, all 3 ablations lead to worse performance compared to \fracssm{}.
These results further validate the importance of the fractional parameterization, the geometric timescale grid, and the normalized read/write routing.

\paragraph{Finite SoE approximation quality}
To directly evaluate the approximation quality of the finite mode bank, we consider the homogeneous fractional relaxation equation~\eqref{eq:fractional-relaxation} with $u(t)=0$ and $h(0)=1$, whose exact solution is $h(t)=E_{\alpha}(-\lambda t^\alpha)$. 
Using the normalized coordinate $s=\lambda^{1/\alpha}t$ we approximate $E_{\alpha}(-s^\alpha)$ using positive mixtures of the same geometrically
spaced timescales $\tau\in[1,2^{17}]$ as in \fracssm{}. The resulting range for the normalized time is $s\in[4,6553.6]$ because we use the normalized step $\Delta\lambda^{1/\alpha}=0.1$, so the $64$K-token horizon corresponds to $s=0.1\times65536=6553.6$, and we start at $s=4\tau_{\min}=4$ to focus on the regime covered by the exponential bank.
For each $\alpha\in\{0.10,0.11,\ldots,0.99\}$, we optimize the coefficients
$c_m$ of the fixed exponential bank under the simplex constraint
$c_m\geq0$, $\sum_m c_m=1$, minimizing the maximum absolute error of SoE approximation:
$$ 
\varepsilon_{\alpha,M}
=
\max_{s}
\left|
E_{\alpha}(-s^\alpha)
-
\sum_{m=1}^{M} c_m e^{-s/\tau_m}
\right|.
$$ 
We repeat this for $M\in\{8,16,32\}$ and report the mean of
$\varepsilon_{\alpha,M}$ over all values of $\alpha$ in Table~\ref{tab:finite-bank-approx}. The $M=16$ bank used in our experiments approximately halves the mean maximum error relative to $M=8$, while increasing to $M=32$ provides only a marginal further improvement.

\begin{table*}[h]
\centering
\begin{tabular}{lccc}
\toprule
Modes ($M$) & 8 & 16 & 32 \\
\midrule
Mean maximum error
& $1.34\times10^{-3}$
& $6.82\times10^{-4}$
& $6.45\times10^{-4}$ \\
\bottomrule
\end{tabular}
\caption{Approximation error of the finite geometric mode bank.}
\label{tab:finite-bank-approx}
\end{table*}

\section{Limitations and Broader Impact}
\label{app:limitations-impact}

\subsection{Limitations}
\label{app:limitations}

Our work has several limitations. First, \fracssm{} realizes the fractional kernel through a finite sum-of-exponentials approximation over a fixed geometric bank of timescales. This makes the model practical and scan-compatible, but the approximation quality depends on the number of modes and on whether the selected timescale range covers the dependencies required by the task. In applications where the relevant memory scale lies outside this range, or where the task is dominated by short-tailed dependencies requiring rapid forgetting, the advantage of the method may be reduced. 
Second, the computational cost of \fracssm{} scales with the number of modes \(M\). The transition within each mode is simple and diagonal, but the layer maintains a bank of \(M\) recurrent states and performs mode-wise read and write operations at every step. Thus, while the method remains linear in sequence length and bounded-state, efficient large-scale use requires care in kernel design and memory layout.
Finally, our empirical evaluation focuses on settings where long-range dependencies are central, including synthetic long-memory tasks and long-context language modeling. While these experiments directly test the intended use case of \fracssm{}, they do not exhaust the possible domains where fractional memory may be useful. Evaluating \fracssm{} in additional long-horizon settings, such as video, scientific sequences, and persistent memory for agentic systems, is an important direction for future work.

\subsection{Broader Impact}
\label{app:broader-impact}
The potential positive impact of this work is to improve the efficiency and long-context capability of sequence models by providing a bounded-state recurrent architecture with an improved memory law. This may support applications where long-range dependencies are important, including language modeling, scientific sequences, time-series data, video, and persistent memory for agentic systems. At the same time, stronger long-context models can inherit the usual risks of capable sequence models, including misuse in text or code generation, surveillance, automated decision support, and systems that reproduce biases or private information present in training data. Future releases of large pretrained checkpoints based on \fracssm{} should therefore include application-appropriate safeguards, such as dataset review, misuse evaluation, privacy protections, and access controls when needed.

\end{document}